\documentclass{article} 

\usepackage{xcolor}  
\usepackage[breaklinks=true,colorlinks]{hyperref}
\hypersetup{
  linkcolor={red!50!black},
  citecolor={blue!50!black},
}   
\usepackage{url}

\usepackage[utf8]{inputenc} 
\usepackage[truedimen,margin=20mm]{geometry} 

\usepackage[T1]{fontenc}    
\usepackage{url}            
\usepackage{booktabs}       
\usepackage{amsfonts}       
\usepackage{nicefrac}       
\usepackage{microtype}      
\usepackage[pdftex]{graphicx}
\usepackage{natbib} 
\usepackage{algorithm} 
\usepackage{algorithmic} 
\usepackage{multirow}
\usepackage{amsmath}
\usepackage{amssymb}
\usepackage{amsthm}
\usepackage{here}
\usepackage{wrapfig}
\usepackage{centernot}
\usepackage{enumitem}
\usepackage{comment}
\usepackage{subfig}
\usepackage{bbm}
\usepackage{titlesec}
\usepackage{abstract} 

\allowdisplaybreaks[1]
\titleformat*{\section}{\large\bfseries}

\newtheorem{theorem}{Theorem}
\newtheorem{lemma}{Lemma}
\newtheorem{proposition}{Proposition}

\newtheorem{assumption}{Assumption}
\newtheorem{example}{Example}

\usepackage{color}

\newcommand{\defeq}{\overset{\mathrm{def}}{=}}

\newcommand{\tr}{\mathrm{Tr}}
\newcommand{\im}{\mathrm{Im}}

\def\pd<#1>{\left\langle #1 \right\rangle}
\def\floor[#1]{\left\lfloor #1 \right\rfloor}
\def\ceil[#1]{\left\lceil #1 \right\rceil}

\newcommand{\rmd}{\mathrm{d}}

\newcommand{\bE}{\mathbb{E}}

\newcommand{\bP}{\mathbb{P}}
\newcommand{\bR}{\mathbb{R}}
\newcommand{\bZ}{\mathbb{Z}}

\newcommand{\cF}{\mathcal{F}}

\title{\Large Why is parameter averaging beneficial in SGD? An objective smoothing perspective}

\author{Atsushi Nitanda$^{1,2\dag}$, Ryuhei Kikuchi$^{3}$, Shugo Maeda$^{3}$, Denny Wu$^{4,5}$
\vspace{2mm}\\
\normalsize{\textit{$^1$IHPC, Agency for Science, Technology and Research, Singapore}} \\
\normalsize{\textit{$^2$CFAR, Agency for Science, Technology and Research, Singapore}} \\
\normalsize{\textit{$^3$Kyushu Institute of Technology, Japan}} \\
\normalsize{\textit{$^4$New York University, United States}}~~~
\normalsize{\textit{$^5$Flatiron Institute, United States}} \\
\small{Email: $^\dag$atsushi\_nitanda@cfar.a-star.edu.sg}} 

\date{}

\begin{document}
\twocolumn
\maketitle

\begin{abstract}
It is often observed that stochastic gradient descent (SGD) and its variants implicitly select a solution with good generalization performance; such implicit bias is often characterized in terms of the \textit{sharpness} of the minima. \citet{kleinberg2018alternative} connected this bias with the smoothing effect of SGD which eliminates sharp local minima by the convolution using the stochastic gradient noise. We follow this line of research and study the commonly-used averaged SGD algorithm, which has been empirically observed in \citet{izmailov2018averaging} to prefer a flat minimum and therefore achieves better generalization. We prove that in certain problem settings, averaged SGD can efficiently optimize the smoothed objective which avoids sharp local minima. In experiments, we verify our theory and show that parameter averaging with an appropriate step size indeed leads to significant improvement in the performance of SGD.  
\end{abstract}
\section{INTRODUCTION}
Stochastic gradient descent (SGD) \citep{robbins1951stochastic} is a powerful learning method for training deep neural networks. SGD often exhibits higher generalization performance than many of its variants, even when they achieve faster convergence with respect to the training loss \citep{keskar2017improving,wilson2017marginal,luo2019adaptive}. Therefore, the study of {\it implicit bias} to characterize the parameters obtained by SGD has become an active research topic.

Among such studies, {\it flat minima} \citep{hinton1993keeping,hochreiter1997flat} has been recognized as an important notion relevant to the generalization performance of deep neural networks. \citet{hochreiter1997flat,keskar2017large} suggested that the flat minima generalize well compared to sharp minima, and \citet{neyshabur2017exploring} rigorously supported this correlation under $\ell_2$-regularization by using the PAC-Bayesian framework \citep{mcallester1998some,mcallester1999pac}. \citet{jiang2020fantastic} verified that the flatness measures reliably capture the generalization performance compared to many other complexities. Furthermore, \citet{keskar2017large} empirically demonstrated that SGD prefers a flat minimum due to its own stochastic gradient noise and \citet{kleinberg2018alternative} proved this implicit bias via the smoothing effect brought by the noise.

Along this line of research, there have been several attempts to develop optimization methods with stronger bias than SGD. Especially, stochastic weight averaging (SWA) \citep{izmailov2018averaging} and sharpness aware minimization (SAM) \citep{foret2020sharpness} achieved significant improvement in generalization performance over SGD. SWA is a cyclic averaging scheme for SGD, which includes the averaged SGD \citep{ruppert1988efficient,polyak1992acceleration} as a special case. It is well known that averaged SGD with an appropriately small step size or diminishing step size is a statistically optimal method for the convex optimization problems \citep{bach2011non,lacoste2012simpler,rakhlin2012making}. On the other hand, \citet{izmailov2018averaging} found that an appropriately large step size is rather preferable to a small step size when training deep neural networks with the averaged SGD.

\begin{figure*}[t]
\center
\begin{tabular}{ccc}
\begin{minipage}[t]{0.3\linewidth}
\centering
\includegraphics[width=\textwidth]{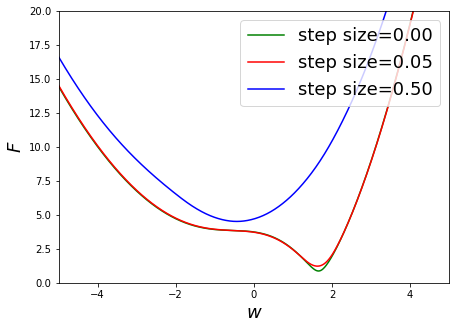} \\
\end{minipage} &
\begin{minipage}[t]{0.3\linewidth}
\centering
\includegraphics[width=\textwidth]{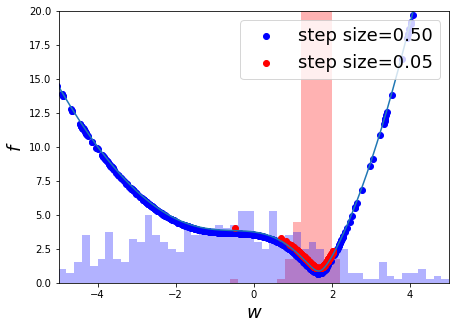} \\
\end{minipage} &
\begin{minipage}[t]{0.3\linewidth}
\centering
\includegraphics[width=\textwidth]{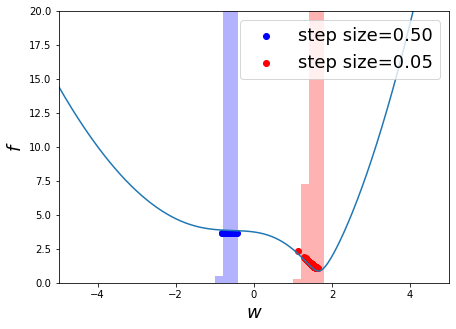} \\
\end{minipage} \\
\vspace{-2mm}
\begin{minipage}[t]{0.3\linewidth}
\centering
\includegraphics[width=\textwidth]{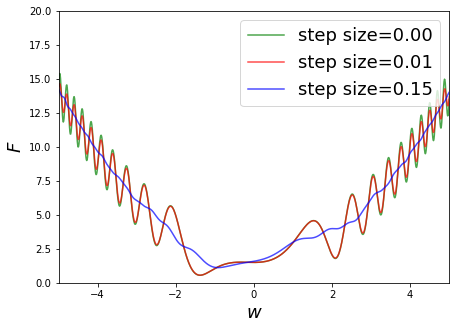} \\
\small (a)~(smoothed) objective
\end{minipage} &
\begin{minipage}[t]{0.3\linewidth}
\centering
\includegraphics[width=\textwidth]{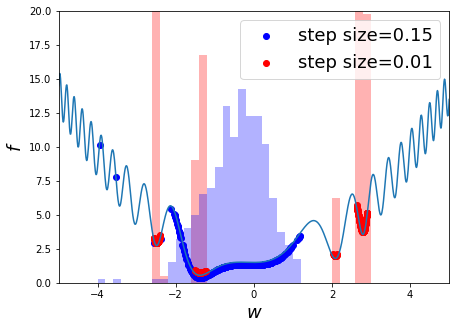} \\
\small (b)~SGD
\end{minipage} &
\begin{minipage}[t]{0.3\linewidth}
\centering
\includegraphics[width=\textwidth]{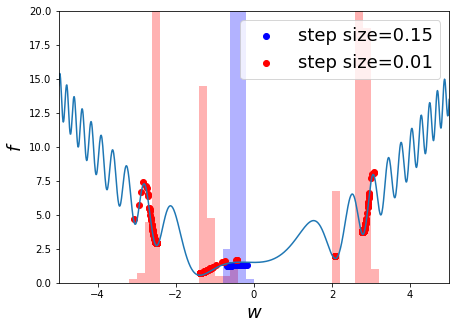} \\
\small (c)~averaged SGD
\end{minipage}
\end{tabular}
\caption{We run SGD and averaged SGD 500 times with the uniform stochastic gradient noise for two objective functions (top and bottom). Figure (a) depicts the objective function $f$ (green, $\eta=0$) and smoothed objectives $F$ (red and blue, $\eta>0$). Figures (b) and (c) plot convergent points by SGD and averaged SGD with histograms, respectively.}
\label{fig:alg_comparison}
\end{figure*}

The success of using a large step size can be attributed to the strong bias towards a flat minimum as discussed in \citet{izmailov2018averaging}. SGD with a large step size cannot stay in sharp regions because of the amplified stochastic gradient noise, and thus it moves to a flatter region. After a long run, SGD will finally oscillate according to an invariant distribution. Then, by taking the average, we can get the mean of this distribution, which is located inside a flat region. Although this provides a good insight into how the averaged SGD with a large step size behaves, the theoretical understanding remains elusive. Hence, the research problem we aim to address is  

\vspace{-1mm}
\begin{center}
{\it When and why does the averaged SGD converge to a flat region more stably than SGD?}
\end{center}
\vspace{-1mm}

In our work, we address this question by analyzing the averaged SGD based on the smoothing effect of stochastic gradient noise.

\subsection{Contributions}
We employ the {\it alternative view} of SGD \citep{kleinberg2018alternative} that connects SGD for the objective $f(w)$ with the optimization of the smoothed objective $F(v) = \bE[ f(v- \eta \epsilon') ]$ through the change of variable $w \mapsto v = w- \eta \nabla f(w)$, where $\epsilon'$ is a stochastic gradient noise. In fact, $F$ is a smoothed function because it is basically the convolution using the stochastic gradient noise, and the strength of the smoothness is controlled by the step size $\eta$ as seen in the left figures of Figure \ref{fig:alg_comparison}.

Therefore, an optimization method that can minimize $F$ with a certain accuracy is expected to converge to a flat region, and our aim is to clarify {\it when the averaged SGD can get closer to the minimizer of $F$ than SGD under the same step size.}
Figures \ref{fig:alg_comparison} and \ref{fig:var_eta} illustrate toy examples, respectively, where the averaged SGD works better in the above sense especially when using a large step size and where the averaged SGD can approximately optimize $F$ with $\eta$ varying.

\begin{figure}[ht]
\centering
\vspace{-2mm}
\includegraphics[width=0.8\linewidth]{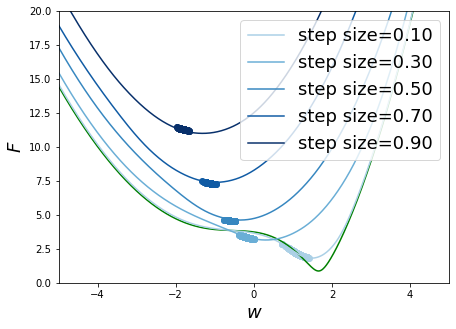}
\caption{The figure plots the original objective (green), smoothed objectives (blue, darker is smoother), and convergent points obtained by the averaged SGD which is run 500 times for each step size $\eta \in \{0,1,~0.3,~0.5,~0.7,~0.9\}$.} 
\label{fig:var_eta}
\end{figure}

Our contributions are summarized below: 
\begin{itemize}
    \setlength{\leftskip}{-3mm}
    \item We derive the upper and lower bound conditions in Theorem \ref{thm:asgd} and Eq.~(\ref{eq:improvement_condition}) for the SGD error: $D_\infty = \lim_{T\rightarrow \infty}\sqrt{\frac{1}{T+1}\sum_{t=0}^T \|v_t - v_*\|^2}$ such that averaged SGD gets closer to the minimizer $v_*=\arg\min F(v)$ than SGD.
    \item We estimate the SGD error $D_\infty$ in Proposition \ref{prop:sgd_upper_bound} and \ref{prop:sgd_lower_bound} under additional lower bound on the stochastic gradient noise and sort of one-point strong convexity for $F$, with a specific Example \ref{example:sanity_check} satisfying the required condition.
    \item We empirically observe that averaged SGD achieves high test accuracy on image classification tasks when using a relatively large step size that makes SGD itself unstable.
\end{itemize}

\section{PRELIMINARY}
In this section, we introduce the stochastic gradient descent (SGD) for general problems including the risk minimization problems appearing in machine learning, and introduce the alternative view of SGD developed by \cite{kleinberg2018alternative}. 

\subsection{Stochastic gradient descent}\label{subsec:sgd}
Let $f: \bR^d \rightarrow \bR$ be a smooth nonconvex objective function to be minimized. For simplicity, we assume $f$ is nonnegative.
A stochastic gradient descent, randomly initialized at $w_0$, for optimizing $f$ is described as follows: for $t = 0, 1, 2, \ldots$
\begin{equation}\label{eq:sgd}
    w_{t+1} = w_t - \eta \left( \nabla f (w_t) + \epsilon_{t+1}(w_t) \right), 
\end{equation}
where $\eta > 0$ is the step size and $\epsilon_{t+1}: \bR^d \rightarrow \bR^d$ is a random field corresponding to the stochastic gradient noise i.e., for any $w \in \bR^d$, $\{ \epsilon_{t+1}(w)\}_{t=0}^\infty$ is a sequence of zero-mean random variables taking values in $\bR^d$. 
A typical setup of the above is an empirical risk minimization in machine learning.

\begin{example}\label{example:risk_mimimization}
Let $\ell(w,z)$ be a loss function where $w \in \bR^d$ and $z \in \bR^p$ represent the parameter and data example, respectively, and let $r(w)$ be a regularization function. Let $\{z_i\}_{i=1}^n$ be training examples. Then, the regularized empirical risk is defined as follows:
\[ f(w) = \frac{1}{n}\sum_{i=1}^n \ell(w,z_i) + \lambda r(w), \]
where $\lambda > 0$ is a regularization coefficient. A standard stochastic gradient at $t$-th iterate $w_t$ is defined as $\nabla_w \ell( w_t, z_{i_{t+1}}) + \lambda \nabla r(w)$ where $\{i_{t+1}\}_{t=0}^\infty$ are i.i.d. random variables following the uniform distribution on $\{1,\ldots,n\}$.
Hence, the stochastic noise is $\epsilon_{t+1}(w) = \nabla_w \ell(w, z_{i_{t+1}}) - \frac{1}{n}\sum_{i=1}^n \nabla_w \ell (w,z_i)$. 
Furthermore, we note that the above setup can extend to the case involving the data augmentation by replacing the loss $\ell(w,z_i)$ with $\bE[\ell(w,Z_i)]$, where $Z_i$ is a random variable of an augmented example, and replacing the stochastic gradient of the loss $\nabla_w \ell( w_t, z_{i_{t+1}})$ with $\nabla_w \ell( w_t, Z_{i_{t+1}})$.
\end{example}

\subsection{Alternative view of SGD}\label{subsec:alt_view}
An alternative view \citep{kleinberg2018alternative} of SGD is key in our analysis connecting the stochastic gradient noise with the smoothing for the objective function. 

This view of SGD translates $\{w_t\}_{t=0}^\infty$ into associated iterations $\{v_t\}_{t=0}^\infty$ and we analyze the update of $v_t$ instead of $w_t$. We define $v_t$ as a parameter obtained by the exact gradient descent from $w_t$, that is, $v_t = w_t - \eta \nabla f(w_t)$.
Since $w_{t+1} = v_t - \eta \epsilon_{t+1}(w_t)$, we get $v_{t+1} = v_t - \eta \epsilon_{t+1}(w_t) - \eta \nabla f( v_t - \eta \epsilon_{t+1}(w_t) )$, where the notation $\nabla f( v_t - \eta \epsilon_{t+1}(w_t))$ represents $\nabla f(w)|_{w = v_t - \eta \epsilon_{t+1}(w_t)}$.
As shown in Appendix \ref{subsec:isgd}, under a specific setting given later, the translation $w \mapsto v=w - \eta \nabla f(w)$ is a smooth invertible injection and its inverse is differentiable. Thus, we can identify $\epsilon'_{t+1}(v)$ with $\epsilon_{t+1}(w)$ through the map $w\mapsto v$.
Then, we get an update rule of $v_t$:
\begin{equation}\label{eq:implicit_sgd}
v_{t+1} = v_t - \eta \epsilon'_{t+1}(v_t) - \eta \nabla f( v_t - \eta \epsilon'_{t+1}(v_t) ).
\end{equation}
For convenience, we refer to the rule (\ref{eq:implicit_sgd}) as an {\it alternative stochastic gradient descent} in this paper.
Since, the conditional expectation of $\epsilon_{t+1}'(v_t)$ at $v_t$ is zero, we expect that the alternative SGD (\ref{eq:implicit_sgd}) minimizes the following smoothed objective function:
\begin{equation}\label{eq:smoothed_objective}
    F(v) = \bE[ f( v - \eta \epsilon'(v)) ],
\end{equation}
where $\epsilon'$ is an independent copy of $\epsilon'_1, \epsilon'_2, \ldots$.
However, we note that the alternative SGD is a biased SGD because $\nabla F(v_t) \neq \bE[\nabla f( v_t - \eta \epsilon_{t+1}'(v_t) )]$ in general
\footnote{We see by the chain rule $\nabla F(v) = \bE[ (I - \eta J_{\epsilon'(v)}^\top(v))\nabla f( v - \eta \epsilon'(v))]$. Here, the notation $\nabla f( v - \eta \epsilon'(v) )$ represents $\nabla f(w)|_{w= v - \eta \epsilon'(v)}$. Hence, $\nabla F(v) \neq \bE[ \nabla f( v - \eta \epsilon'(v))]$ unless $J_{\epsilon'(v)}^\top(v)=0$.}, unless $\epsilon_{t+1}'(v)$ is free from $v$. 

The function (\ref{eq:smoothed_objective}) is indeed a smoothed function of $f$ by the convolution using the stochastic gradient noise $\eta \epsilon'$, where the step size $\eta$ controls the strength of smoothness.
Figure \ref{fig:alg_comparison} depicts how a nonconvex function $f$ is smoothened. In this figure, we observe that by using an appropriately large step size, sharp local minima of $f$ vanishes and the solution of $F$ emerges at the flat area of $f$. Moreover, by taking Taylor expansion of $f$, we see that $F(v)$ is a regularized objective that penalizes the high (positive) curvature of $f$ along the noise direction in expectation:
\begin{equation*}
    F(v) = f(v) + \frac{\eta^2}{2} \tr\left(\nabla^2 f(v) \bE[\epsilon'(v) \epsilon'(v)^\top]\right) + O(\eta^3).
\end{equation*}
Therefore, we can expect that stochastic gradient descent avoids sharp minima and converges to a flat region.

\citet{kleinberg2018alternative} showed the convergence to a point $v_\circ \in \bR^d$ under the regularity condition: $\bE[\nabla f( v - \eta \epsilon'(v) )]^\top (v-v_\circ) \geq \exists c \|v-v_\circ\|^2$ via the alternative view of SGD.
Inspired by their work, we analyze the optimization capability of the averaged SGD for approximating the minimizer $v_* = \arg\min_{v \in \bR^d} F(v)$.

\section{CONVERGENCE ANALYSIS}\label{sec:asgd}
\citet{izmailov2018averaging} empirically demonstrated that averaged SGD converges to a flat region and achieves superior generalization when using a relatively large step size such that SGD oscillates.
We theoretically explain this phenomenon by showing that the averaged SGD can get closer to $v_*$ than SGD using the same step size under certain settings. 
In the averaged SGD, we run SGD (\ref{eq:sgd}) and take the average as follows: $\overline{w}_{T+1} = \frac{1}{T+1}\sum_{t=1}^{T+1} w_t$.

\subsection{Analysis of averaged SGD}
Our aim is to show $\lim_{T\rightarrow \infty}\overline{w}_T$ can be closer to $v_*$ than $\{w_t\}_{t=0}^\infty$ and $\{v_t\}_{t=0}^\infty$. 
Preferably, the alternative view of SGD (\ref{eq:implicit_sgd}) is more useful in analyzing the averaged SGD because the average $\overline{v}_T=\frac{1}{T+1}\sum_{t=0}^T v_t$ is consistent with $\overline{w}_T$ as confirmed below.
By definition,  
\[ \overline{w}_{T+1} = \overline{v}_T + \frac{\eta}{T+1}\sum_{t=0}^{T}\epsilon_{t+1}(w_t), \]
where the noise term $\nicefrac{\sum_{t=0}^{T}\epsilon_{t+1}(w_t)}{(T+1)}$ is zero in expectation and its variance is upper bounded by $\nicefrac{\sigma_1^2}{(T+1)}$ under Assumption {\bf (A2)}. Hence, $\overline{w}_{T+1} - \overline{v}_{T}$ converges to zero in probability by Chebyshev's inequality; for any $r>0$, $\bP[ \|\overline{w}_{T+1} - \overline{v}_{T}\| > r ] \leq \nicefrac{\sigma_1^2}{(T+1)r^2} \rightarrow 0$ as $T \rightarrow \infty$, and the analysis of $\lim_{T\rightarrow \infty} \overline{w}_{T}$ reduces to that of $\lim_{T\rightarrow \infty} \overline{v}_T$.

We can immediately see that the averaged SGD is always closer to the solution than SGD in average because of Jensen's inequality: 
$\bE[\|  \overline{v}_T - v_* \|^2] \leq \frac{1}{T+1}\sum_{t=0}^T \bE[\|  v_t - v_* \|^2]$.
We next present a nontrivial bound on the left-hand side of this inequality (i.e., the error achieved by the averaged SGD).
To do so, we make the following assumptions on the objective function and stochastic gradient noise.
\begin{assumption}\label{assump:asgd}\ 
\begin{itemize}[topsep=0mm,itemsep=0mm]
    \item[{\bf(A1)}] $f: \bR^d \rightarrow \bR$ is nonnegative, twice continuously differentiable, and its Hessian is bounded, i.e., there is a constant $L>0$ such that for any $w \in \bR^d$, $-LI \preceq \nabla^2 f(w) \preceq LI$.
    \item[{\bf(A2)}] Random fields $\{\epsilon_{t+1}\}_{t=0}^\infty$ are independent copies each other\footnote{We suppose $\{\epsilon_{t+1}\}_{t=0}^\infty$ are independent copies each other. That is, there is a measurable map from a probability space: $\Omega \ni z \mapsto \epsilon(w,z) \in \bR^d$, and then $\epsilon_{t+1}$ can be written as a measurable map from a product probability space: $\Omega^{\bZ_{\geq 0}} \ni \{z_{s+1}\}_{s=0}^\infty \mapsto \epsilon(w,z_{t+1}) \in \bR^d$ when explicitly representing them.} and $\epsilon_{t+1}(w)$ is differentiable in $w$. Moreover, for any $w \in \bR^d$ $\bE[\epsilon_{t+1}(w) ]=0$ and there are $\sigma_1, \sigma_2 > 0$ such that for any $w \in \bR^d$, $\bE[ \|\epsilon_{t+1}(w)\|^2 ] \leq \sigma_1^2$ and $\bE[\|J_{\epsilon_{t+1}}^\top (w) \|_2 ] \leq \sigma_2$, where $J_{\epsilon_{t+1}}$ is Jacobian of $\epsilon_{t+1}$. 
    \item[{\bf(A3)}] Hessian of $F$ at $v_*$ is positive definite, i.e., there is a positive constant $\mu>0$ such that $\nabla^2 F(v_*) \succeq \mu I$.    
    \item[{\bf(A4)}] There exist positive constants $\gamma$ and $M$ such that for any $v \in \bR^d$, $\|\nabla F(v) - \nabla^2 F(v_*)(v-v_*)\| \leq \gamma + M\|v-v_*\|^2$.
\end{itemize}
\end{assumption}
The smoothness and boundedness conditions {\bf(A1)} on the objective function and the zero-mean and the bounded variance conditions {\bf(A2)} on stochastic gradient noise are commonly assumed in the convergence analysis for the stochastic optimization methods. Standard stochastic gradient noises $\{\epsilon_{t+1}\}_{t=0}^\infty$ defined in Example \ref{example:risk_mimimization} are independent copies. Moreover, if Hessian matrix satisfies $-LI \preceq \nabla^2_w \ell(w,z) \preceq LI$ in Example \ref{example:risk_mimimization}, then the last condition on $J_{\epsilon_{t+1}}$ also holds with at least $\sigma_2 = 2L$ because $J_{\epsilon_{t+1}}(w) = \nabla_w^2 \ell(w,z_{i_{t+1}}) - \nabla^2 f(w)$.

{\bf (A3)} and {\bf (A4)} are assumptions made on the smoothed objective function and used to connect the error of the averaged SGD with that of SGD. {\bf (A3)} ensures the strict positivity of Hessian $\nabla^2 F(v_*)$. {\bf (A4)} measures the linear approximation error of $\nabla F$ at $v_*$. For instance, functions that have bounded third-order derivatives satisfy {\bf (A4)} with $\gamma=0$, which is also assumed in \citet{dieuleveut2020} to show the superiority of the averaging scheme for strongly convex optimization problems. 

The following theorem describes the relationship between the errors achieved by SGD and averaged SGD.
\begin{theorem}\label{thm:asgd}
Under Assumptions {\bf (A1)}--{\bf (A4)}, run the averaged SGD for $T$-iterations with the step size $\eta \leq \frac{1}{2L}$, then $\overline{v}_T$ satisfies the following inequality:
\begin{align*}
    \bE[ &\| \overline{v}_T - v_* \|]
    \leq \min \bigg\{ D_T,~ O\left( T^{-\frac{1}{2}} \right) \\
    &+ \frac{4\sigma_1 \sigma_2 \eta^{\frac{3}{2}} L^{\frac{1}{2}}}{\sqrt{3} \mu} 
    + \frac{2D_{T+1}}{\eta \mu \sqrt{T+1}}
    + \frac{\gamma + M D_T^2}{\mu} \bigg\},
\end{align*}
where we set $D_T := \sqrt{\frac{1}{T+1}\bE\left[\sum_{t=0}^T \left\| v_t - v_* \right\|^2 \right]}$.
\end{theorem}
\begin{proof}[Proof sketch]
We set $R(v) = \nabla F(v) - \nabla^2 F(v_*) (v - v_*)$. 
Taking average of $R(v_t)$ over $t\in \{0,1,\ldots,T\}$, we get
$\nabla^2 F(v_*) (\overline{v}_T - v_*) = \frac{1}{T+1}\sum_{t=0}^T \nabla F(v_t) - \frac{1}{T+1}\sum_{t=0}^T R(v_t).$
Therefore, by taking the expectation of the norm of both sides, we obtain 
\begin{align*}
    \mu \bE\left[\| \overline{v}_T - v_*\|\right] \leq \frac{1}{T+1} \bE\left[ \left\| \sum_{t=0}^T \nabla F(v_t) \right\|\right] 
    + \gamma + MD_T^2,
\end{align*}
where we applied {\bf (A3)} and {\bf (A4)}.
Since the first term of the RHS is the norm of the average of gradients $\nabla F(v_t)$ across all iterations, this term will decrease up to a small error as the optimization proceeds:
\[ O(T^{-1/2}) + \frac{4}{\sqrt{3}} \sigma_1 \sigma_2 \eta^{\frac{3}{2}} L^{\frac{1}{2}} 
    + \frac{2D_{T+1}}{\eta\sqrt{T+1}}. \]
See Lemma \ref{lemma:sgd_appendix_2} in Appendix for the detailed derivation of this error.
Moreover, Jensen's inequality yields $\bE[ \| \overline{v}_T - v_* \|] \leq \sqrt{\bE[ \| \overline{v}_T - v_* \|^2]}  \leq D_T$.
\end{proof}

The term $D_T = \sqrt{\frac{1}{T+1}\bE\left[\sum_{t=0}^T \left\| v_t - v_* \right\|^2 \right]}$ in Theorem \ref{thm:asgd} is the root mean square of the distances between $v_*$ and SGD iterations. Hence $D_T$ converges up to an error depending on $\eta$ under certain assumptions as shown in Proposition \ref{prop:sgd_upper_bound} later or at least it is uniformly bounded as long as SGD performs in a compact domain. 
Hence, we suppose $D_\infty = \lim_{T\rightarrow \infty}D_T < \infty$ in the following argument. 

Let us see the relationship between errors finally achieved by SGD and averaged SGD. Taking the limit $T \rightarrow \infty$, we get the upper bound on $\lim_{T\rightarrow \infty} \bE[\|\overline{v}_T - v_*\|]$ as follows:
\begin{equation} \label{eq:asgd_upper_bound_limit}
 \min\left\{ D_\infty, \frac{4\sigma_1 \sigma_2 \eta^{\frac{3}{2}} L^{\frac{1}{2}}}{\sqrt{3} \mu} + \frac{\gamma + MD_\infty^2}{\mu}\right\}.  
\end{equation}

The bound (\ref{eq:asgd_upper_bound_limit}) shows nontrivial improvement by averaged SGD over SGD in the sense: $\lim _{T\rightarrow \infty} \bE[\|\overline{v}_T - v_*\|] \ll D_\infty$ when SGD error $D_\infty$ satisfies the following condition:
\begin{equation}\label{eq:improvement_condition}
    \frac{4\sigma_1 \sigma_2 \eta^{\frac{3}{2}} L^{\frac{1}{2}}}{\sqrt{3} \mu} + \frac{\gamma}{\mu}
    \ll D_\infty 
    \ll \frac{\mu}{M}.
\end{equation}
Here, the informal notation $a \ll b$ (for $a, b>0$) is used for describing the situation where $b$ is sufficiently large such that $a/b$ approaches $0$. To discuss the condition (\ref{eq:improvement_condition}), consider the case where $\mu,~M>0$ satisfying Assumptions {\bf (A3)} and {\bf (A4)} can be uniformly chosen with respect to $\eta \rightarrow 0$ and $\gamma=0$. Then, the upper bound $D_\infty \ll \mu/M=O(1)$ means the convergence of SGD to some extent. And the lower bound $\Omega(\eta^{3/2}) \ll D_\infty$ is also typically satisfied for mildly small $\eta$ because SGD oscillates at least of the order of $\eta$ if the stochastic gradient does not vanish (see Proposition \ref{prop:sgd_lower_bound}).

\subsection{Evaluation of SGD error $D_\infty$}
In this section, we estimate the value $D_\infty$ under reasonable problem setups.
Inspired by \cite{kleinberg2018alternative}\footnote{{\bf (A5)} is slightly different from that in \cite{kleinberg2018alternative}. They assumed $\bE[\nabla f( v - \eta \epsilon'(v) )]^\top (v-v_\circ) \geq \exists c \|v-v_\circ\|^2$ with a point $v_\circ$ which does not necessarily coincides with $v_*$. If $J_{\epsilon'}\equiv 0$, then $v_\circ$ equals $v_*$ and both assumptions coincide because $\nabla F(v) = \bE[\nabla f( v - \eta \epsilon'(v) )]$ in this case.}, we make the following regularity condition on the smoothed objective $F$ at $v_*$.  

\begin{assumption}\label{assump:sgd2}\ 
{\bf(A5)} There exists $c>0$ such that for any $v \in \bR^d$, $\nabla F(v)^\top (v-v_*) \geq c \|v-v_*\|^2$.
\end{assumption}
{\bf(A5)} is a sort of one-point strong convexity at the solution $v_*$, but we note that it allows for the nonconvexity of $F$.
Under this assumption, we obtain the upper bound on SGD error $D_T$. 

\begin{proposition}\label{prop:sgd_upper_bound}
Under Assumptions {\bf (A1)}, {\bf (A2)}, and {\bf (A5)}, run SGD for $T$-iterations with the step size $\eta \leq \frac{1}{2L}$, then we get
\begin{align*}
    D_T^2 \leq O\left(T^{-1}\right) 
    + \frac{2\eta\sigma_1^2}{c} 
    + \frac{8\eta^2 \sigma_1^2 L}{3c} \left(1 + \frac{2 \eta \sigma_2^2}{c} \right). 
\end{align*}
\end{proposition}
Our result allows for a larger step size than that in \citet{kleinberg2018alternative} because we do not stick to the convergence of the last iterate.
Proposition \ref{prop:sgd_upper_bound} implies $D_\infty \leq \sqrt{\frac{2\eta\sigma_1^2}{c} + \frac{8\eta^2 \sigma_1^2 L}{3c} \left(1 + \frac{2 \eta \sigma_2^2}{c} \right)}$, meaning convergence of SGD to a point at least at a distance of the RHS from the minimizer $v_*$ of $F$.

The next proposition provides the lower bound on SGD error $D_T$ under the assumption that the variance of stochastic gradient noise is uniformly lower bounded.
\begin{proposition}\label{prop:sgd_lower_bound}
Suppose that Assumptions {\bf(A1)} and {\bf(A2)} hold, and that there exists $\sigma_3>0$ such that for any $w\in\bR^d$, $\bE[\|\epsilon_{t+1}(w)\|^2] \geq \sigma_3^2$.
Running SGD with $\eta \leq \frac{3\sigma_3^2}{32\sigma_1^2 L}$, we get
\[ \frac{\eta^2\sigma_3^2}{8} \leq D_T^2 + O(T^{-1}). \]
\end{proposition}

Propositions \ref{prop:sgd_upper_bound} and \ref{prop:sgd_lower_bound} with $T\rightarrow \infty$ yield the estimation of $D_\infty$. That is, we get that under the assumptions made in these two propositions, 
\begin{equation}\label{eq:sgd_estimation}
 \frac{\eta \sigma_3}{2\sqrt{2}} \leq D_\infty \leq \sqrt{\frac{2\eta\sigma_1^2}{c} + \frac{8\eta^2 \sigma_1^2 L}{3c} \left(1 + \frac{2 \eta \sigma_2^2}{c} \right)}.
\end{equation}

Summarizing the discussion so far, the averaged SGD achieves a nontrivial improvement if the above estimation Eq.~(\ref{eq:sgd_estimation}) satisfies the condition Eq.~(\ref{eq:improvement_condition}).

\subsection{Example}
As a sanity check, we present examples that satisfy the improvement condition Eq.~(\ref{eq:improvement_condition}) by instantiating the estimation Eq.~(\ref{eq:sgd_estimation}) of SGD error $D_\infty$ and we confirm the certain improvement by the averaged SGD. For a detailed explanation, see Appendix \ref{sec:example} and \ref{sec:example2}.

\begin{figure}[ht]
\centering
\includegraphics[width=\linewidth]{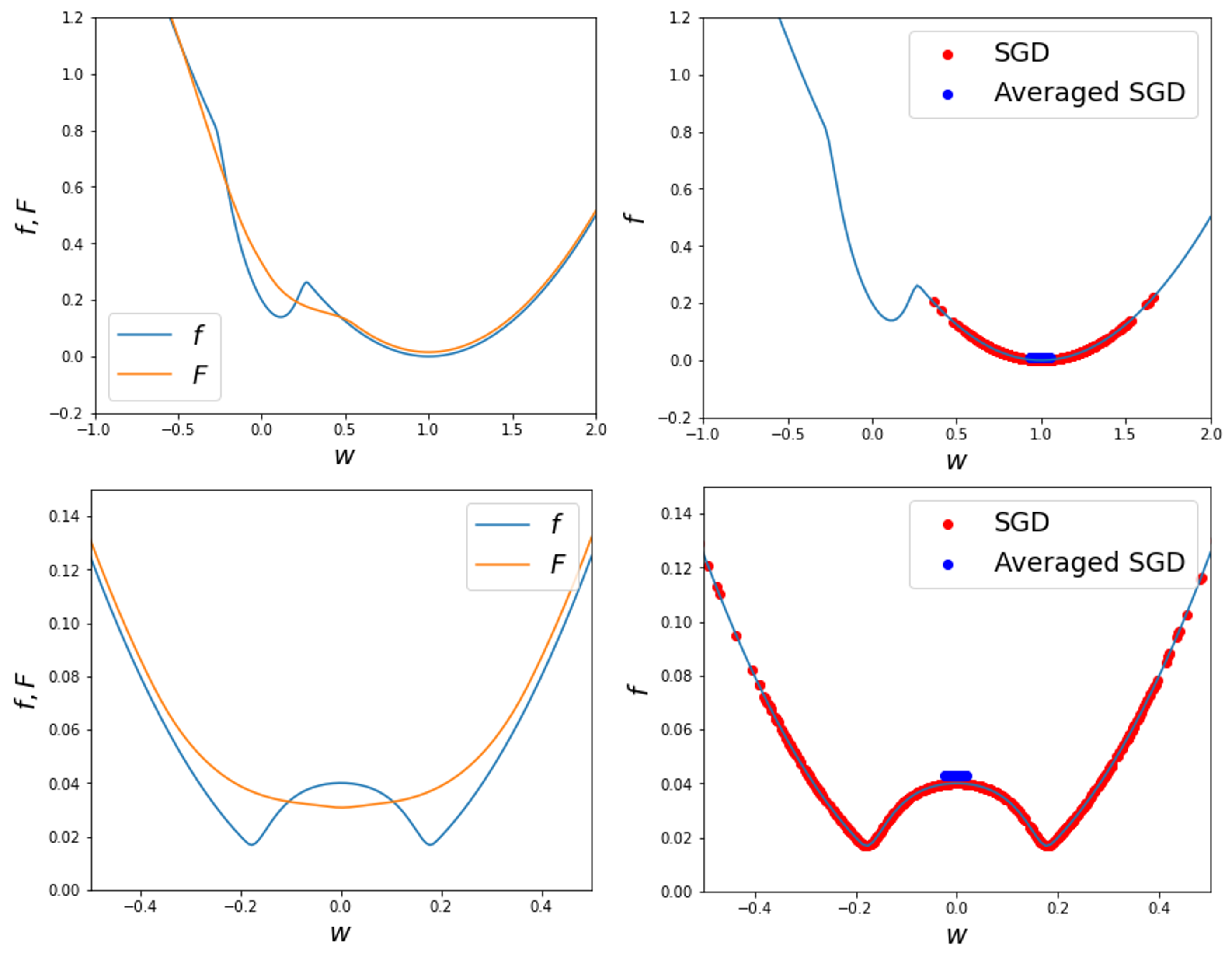}
\caption{We run SGD and averaged SGD for two problems corresponding top and bottom figures. For each case, the left figure depicts the original objective (blue) $f$ and smoothed objective (orange) $F$ and the right figure depicts convergent points of SGD (red) and averaged SGD (blue).} 
\label{fig:sanity_check}
\end{figure} 

\begin{example}\label{example:sanity_check}
We consider two examples of functions $f$ on $\bR$ which fall into the setting of $\gamma=0$ or $M=0$. Blue curves in the top and bottom of Figure \ref{fig:sanity_check} depict these functions, respectively. 
As the stochastic gradient noises, we take i.i.d. uniform noise on an interval $[-r,r]$ in both cases. Hence, $\sigma_1^2=\sigma_3^2=r^2/3$ and $\sigma_2=0$. Then, the smoothed objectives $F$ (orange curves in Figure \ref{fig:sanity_check}) with an appropriate step size $\eta$ look like the quadratic functions that possess the unique global minimizers $v_*$. In both cases, we observe the improved performance of the averaged SGD regarding the convergence to $v_*$. We verify the condition Eq.~(\ref{eq:improvement_condition}) below. For the details, see Appendix.

\begin{enumerate}[topsep=0mm,itemsep=0mm]
\setlength{\leftskip}{-3mm}
\item {\bf First case (top of Figure \ref{fig:sanity_check}):}
    An objective is $f(w)=0.5(w-1)^2 + g_\delta(w)$, where $g_\delta$ is a downward bump of width $\delta$ and height $\delta$. This bump makes a local minimum of $f$ around $v=0$ and $f$ has the optimal solution $v=1$. Lipschitz constant is $L=1+O(1/\delta)$. 
    In this setting, the smoothing only affects the term $g_\delta$, i.e.,
    \[ F(v) = \frac{1}{2}(v-1)^2 + \bE[ g_\delta( v - \eta \epsilon' ) ] + (const).\]
    Then, the local minimum of $f$ around $v=0$ is eliminated by the smoothing with an appropriate $\eta$, and the global solution $v_*=1$ of $F$ only remains. In fact, if we take $\eta = O(\delta/r)$, we can verify all assumptions with constants: $\mu=1,~c=1/3,~\gamma=0,~M=8/9$.
    By instantiating Eq.~(\ref{eq:sgd_estimation}), we see
    \[ \delta \lesssim D_\infty \lesssim \delta + \sqrt{\delta(1+r)}. \]
    Hence, the condition Eq.~(\ref{eq:improvement_condition}): $0 \ll D_\infty \ll 9/8$ is satisfied under the setting with mildly small $\delta$. Specifically, we have $\lim_{T\rightarrow \infty}\bE[\|\overline{v}_T-v_*\|] \leq \frac{8}{9}D_\infty^2$.
\setlength{\leftskip}{-3mm}    
\item {\bf Second case (bottom of Figure \ref{fig:sanity_check}):}
    An objective is $f(w)=0.5w^2 + g_\delta(w)$, where $g_\delta$ is a upward bump of width $\delta$ and height $\delta^2$. $f$ has two global minima beside the bump. Lipschitz constant is $O(1)$ regardless of $\delta$.
    As in the previous example, we see 
    \[ F(v) = \frac{1}{2}v^2 + \bE[ g_\delta( v - \eta \epsilon' ) ] + (const). \]
    Then, two global minima of $f$ are eliminated by the smoothing with an appropriate $\eta$ and a unique global minimum $v_*=0$ of $F$ emerges. In this case, we do not need to make $\eta$ small depending on $\delta$. In fact, as long as $\eta r$ greater than $\delta$ by a certain amount, we can verify all assumptions with constants: $\mu=1,~c=1/2,~\gamma=\exists C\delta^2/(\eta r),~M=0$.
    By instantiating Eq.~(\ref{eq:sgd_estimation}), we obtain the estimation of $D_\infty$ as follows.
    \[ \eta r \lesssim D_\infty \lesssim \sqrt{\eta}r + \eta r. \]
    Hence, the condition Eq.~(\ref{eq:improvement_condition}): $C\delta^2/(\eta r) \ll D_\infty \ll \infty$ is satisfied under the setting $\delta^2 \ll \eta^2 r^2$. Specifically, we have $\lim_{T\rightarrow \infty}\bE[\|\overline{v}_T-v_*\|] \leq C\delta^2/(\eta r)$.
\end{enumerate}
\end{example}


\section{EXPERIMENTS}\label{sec:exp}
We evaluate the empirical performance of SGD and averaged SGD on image classification tasks using CIFAR10 and CIFAR100 datasets. To evaluate the usefulness of the parameter averaging for the other methods, we also compare SAM \citep{foret2020sharpness} with its averaging variant. For the parameter averaging, we employ the tail-averaging scheme where the average is taken over the last phase of training.

\begin{wraptable}[9]{r}[0mm]{37mm}
\vspace{-3mm}
 \caption{Decay schedules for (averaged) SGD.} 
 \label{table:step_size_schedule}
 \centering
 \begin{tabular}{cc}
  \toprule
  $\eta$ & milestones \\ 
  \midrule
  \textit{s} & \{80,160,240\} \\
  \textit{m} & \{80,160\} \\
  \textit{l} & \{300\} \\
  \bottomrule
 \end{tabular}
\end{wraptable}

We use the CNN architectures: ResNet \citep{he2016deep} with 50-layers (ResNet-50), WideResNet \citep{zagoruyko2016wide} with 28 layers and width 10 (WRN-28-10), and Pyramid Network \citep{han2017deep} with 272 layers and widening factor 200. In all settings, we use the standard data augmentations: horizontal flip, normalization, padding by four pixels, random crop, and cutout \citep{devries2017cutout}, and we employ the weight decay with the coefficient $0.005$. Moreover, we use the multi-step strategy for the step size, which decays the step size by a factor once the number of epochs reaches one of the given milestones. To see the dependence on the step size, we use two decay schedules for the parameter averaging.

\begin{table*}[t]
\caption{Comparison of test classification accuracies on CIFAR100 and CIFAR10 datasets.}
\label{table:alg_comparison}
\begin{center}
\begin{footnotesize}
  \begin{tabular}{ccccccccc}
    \toprule
~&
~&
\multicolumn{3}{c}{CIFAR100} & &
\multicolumn{3}{c}{CIFAR10} 
\\
\midrule
~&
$\eta$&
ResNet-50&
WRN-28-10&
Pyramid&
$\eta$&
ResNet-50&
WRN-28-10&
Pyramid
\\
\midrule
SGD&
\textit{s}&
80.83 \scriptsize{(0.21)}&
81.81 \scriptsize{(0.29)}&
81.43 \scriptsize{(0.32)}&
\textit{s}&
95.95 \scriptsize{(0.11)}&
96.85 \scriptsize{(0.16)}&
96.41 \scriptsize{(0.22)}
\\
\addlinespace
\multirow{2}{*}{\begin{tabular}{c}Averaged\\SGD\end{tabular}}&
\textit{s}&
82.13 \scriptsize{(0.22)}&
83.13 \scriptsize{(0.13)}&
84.23 \scriptsize{(0.03)}&
\textit{s}&
96.58 \scriptsize{(0.14)}&
97.24 \scriptsize{(0.07)}&
97.07 \scriptsize{(0.08)}
\\
&
\textit{l}&
\textbf{82.87 \scriptsize{(0.13)}}&
\textbf{84.23 \scriptsize{(0.10)}}&
\textbf{85.12 \scriptsize{(0.20)}}&
\textit{m}&
\textbf{96.89 \scriptsize{(0.05)}}&
\textbf{97.44 \scriptsize{(0.04)}}&
\textbf{97.28 \scriptsize{(0.13)}}
\\
\midrule
SAM&
\textit{s}&
82.56 \scriptsize{(0.14)}&
83.80 \scriptsize{(0.27)}&
84.59 \scriptsize{(0.24)}&
\textit{s}&
\textbf{96.34 \scriptsize{(0.12)}}&
97.14 \scriptsize{(0.05)}&
97.34 \scriptsize{(0.03)}
\\
\addlinespace
\multirow{2}{*}{\begin{tabular}{c}Averaged\\SAM\end{tabular}}&
\textit{s}&
82.64 \scriptsize{(0.12)}&
84.09 \scriptsize{(0.30)}&
85.40 \scriptsize{(0.12)}&
\textit{s}&
96.33 \scriptsize{(0.10)}&
\textbf{97.21 \scriptsize{(0.05)}}&
97.34 \scriptsize{(0.03)}
\\
&
\textit{l}&
\textbf{82.73 \scriptsize{(0.28)}}&
\textbf{84.55 \scriptsize{(0.17)}}&
\textbf{86.00 \scriptsize{(0.04)}}&
\textit{m}&
96.31 \scriptsize{(0.11)}&
97.20 \scriptsize{(0.06)}&
\textbf{97.35 \scriptsize{(0.06)}}
\\
\bottomrule
\end{tabular}
\end{footnotesize}
\end{center}
\end{table*}

Table \ref{table:step_size_schedule} summarizes milestones labeled by the symbols: \textit{`s'}, \textit{`m'}, and \textit{`l'}. The initial step size and a decay factor of the step size are set to $0.1$ and $0.2$ in all cases. The averages are taken from 300 epochs for the schedules \textit{`s'} and \textit{`l'}, and from 160 epochs for the schedule \textit{`m'}. These hyperparameters were tuned based on the validation performance.

For a fair comparison, we run (averaged) SGD with 400 epochs and (averaged) SAM with 200 epochs because SAM requires two gradients per iteration, and thus the milestones and starting epoch of taking averages are also halved for (averaged) SAM. We evaluate each method 5 times for ResNet-50 and WRN-28-10, and 3 times for Pyramid network. The averages of classification accuracies are listed in Table \ref{table:alg_comparison} with the standard deviations in brackets. We observe from the table that the parameter averaging for SGD improves the classification accuracies in all cases, especially on CIFAR100 dataset. Eventually, the averaged SGD achieves comparable performance with SAM. Moreover, we also observe improvement by parameter averaging for SAM in most cases, which is consistent with the observations in \citet{kaddour2022}. A similar improvement is also observed for the cosine annealing learning rate (see Appendix \ref{sec:add_exp}).

\begin{figure}[ht]
\centering
\includegraphics[width=0.8\linewidth]{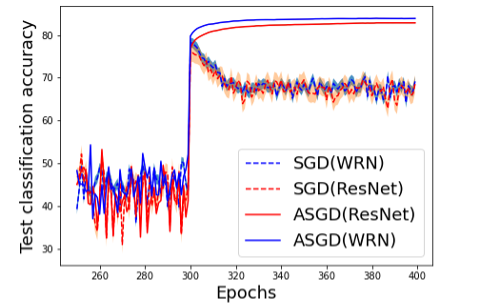}
\caption{Test accuracies achieved by SGD and averaged SGD on CIFAR100 with ResNet-50 and WRN-28-10.}
\label{fig:cifar100_test_curves}
\end{figure}
Comparing results on CIFAR100 and CIFAR10, the large step size is better, and the small step size is relatively poor on CIFAR100 dataset, whereas the small step size generally works on CIFAR10 dataset. If we use the step-size strategy \textit{`l'} for CIFAR10, then the improvement becomes small (see Appendix \ref{sec:add_exp} for this result). This is because the strong smoothing brought by a large step size can be harmful to simple datasets such that the normal SGD already achieves high accuracies. Moreover, we note that the averaged SGD on CIFAR100 quite works well with the large step-size schedule \textit{`l'}, even though SGD itself is unstable and poorly performs under this schedule as seen in Figure \ref{fig:cifar100_test_curves}. Indeed, the accuracy of SGD temporarily increases at the 300 epochs because of the decay of the step size, and it decreases thereafter. However, the parameter averaging brings significant improvement in accuracy even under such a situation.

\begin{figure*}[t]
\center
\begin{tabular}{ccc}
\begin{minipage}[t]{0.31\linewidth}
\centering
\includegraphics[width=\textwidth]{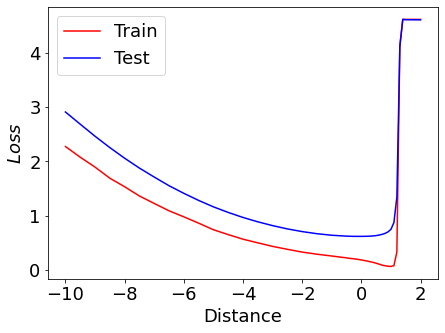} \\
\end{minipage} 
\hspace{1mm}
\begin{minipage}[t]{0.31\linewidth}
\centering
\includegraphics[width=\textwidth]{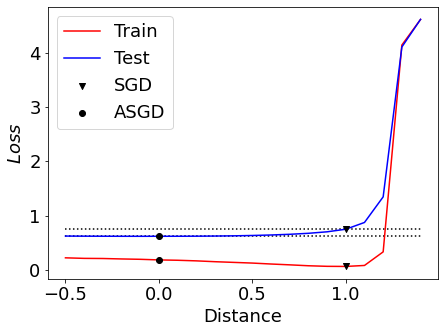} \\
\end{minipage} 
\begin{minipage}[t]{0.33\linewidth}
\centering
\includegraphics[width=\textwidth]{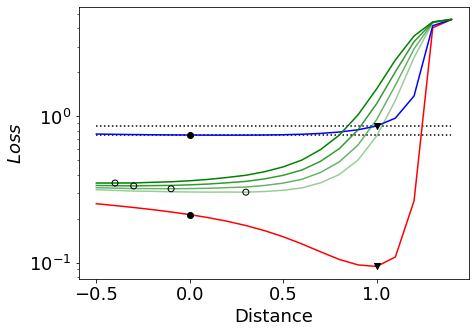} \\
\end{minipage} 
\end{tabular}
\caption{Sections of the train (red) and test (blue) loss landscapes across the parameters obtained by averaged SGD (distance=$0$) and SGD (distance=$1$) for ResNet-50 with CIFAR100 dataset. SGD is run with a small step size after running averaged SGD with a large step size. The middle figure is the close-up view at the edge. The triangle and circle markers represent convergent parameters by SGD and averaged SGD, respectively. The right figure plots smoothed train loss functions (green, darker is smoother) with Gaussian noises in addition to train and test losses. The blank circles are the minimizers of smoothed objectives.}
\label{fig:asymmetric_valley}
\end{figure*}

Next, we observe in Figure \ref{fig:asymmetric_valley} that the loss landscape around the convergent point is in better shape and forms an asymmetric valley as observed by \cite{he2019asymmetric}. Specifically, Figure \ref{fig:asymmetric_valley} depicts the section of train and test loss functions across parameters obtained by the averaged SGD and SGD.  The middle figure is the close-up view at the edge and plots each parameter. The right figure depicts the train and test losses, and the smoothed train losses with Gaussian noises in log-scale.
We observe in Figure \ref{fig:asymmetric_valley} the phenomenon that SGD converges to an edge and averaged SGD converges to a flat side. Our theory helps explain this phenomenon because the minimizer of the asymptotic valley can be shifted to a flat side by the smoothing as confirmed in the right of Figure \ref{fig:asymmetric_valley} as well as a synthetic setting (Figure \ref{fig:var_eta}). Moreover, the right figure indicates the possibility that the smoothed objective with appropriate stochastic gradient noise well approximates test loss. In fact, we observe that averaged SGD achieves a lower test loss which makes about 2\% improvement in the classification error on CIFAR100 dataset. These observations are consistent with the experiments conducted in \citet{he2019asymmetric}.

\begin{figure}[!h]
\centering
\includegraphics[width=\linewidth]{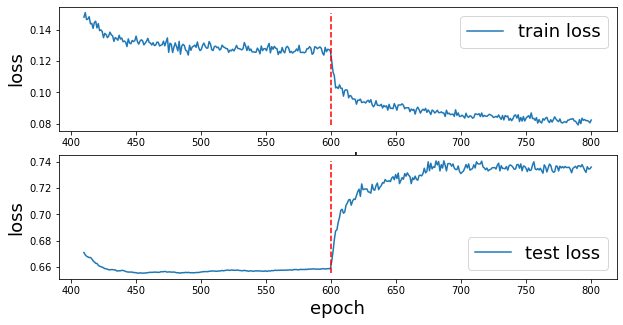}
\caption{For ResNet-50 with CIFAR100 dataset, we run SGD for 200 epochs from a parameter obtained by the averaged SGD with 600 epochs. The figure depicts training loss (top) and test loss (bottom). The red line is a change point of the methods. The learning rate for SGD is eventually annealed to $0$ from $0.02$ used for the final phase of averaged SGD.}
\label{fig:loss_plot}
\end{figure}

Finally, to observe the non-trivial bias of the averaged SGD, we run SGD with the smaller step size from the parameter obtained by the averaged SGD with 600 epochs. Then, we observe in Figure \ref{fig:loss_plot} that the test loss is gradually getting worse, whereas the training loss is getting better. This result proves that the averaged SGD with a large step size does not necessarily converge to a local minimum of the training loss, but converges to a parameter achieving a smaller test loss. Moreover, we confirm that the test accuracy also decreases from 83.35 attained by the averaged SGD to 82.42.

\section*{CONCLUSION}
\cite{izmailov2018averaging} observed the averaged SGD with a large step size finds a flatter solution than SGD and gave an intuitive explanation of this phenomenon. To support their observation, we analyzed the convergence capability of the averaged SGD for the smoothed objective $F$ via the alternative view of SGD \citep{kleinberg2018alternative}. Specifically, we derived the upper and lower bound conditions (Theorem \ref{thm:asgd} and Eq.~(\ref{eq:improvement_condition})) for SGD error $D_\infty$ so that the averaged SGD converges closer to $v_* = \arg\min F(v)$ than SGD. Furthermore, to verify these conditions, we estimated $D_\infty$ under additional lower bound on the stochastic gradient noise and sort of one-point strong convexity for $F$. Finally, we empirically observed that the averaged SGD with a large step size achieved superior performance on the image classification task and confirmed that the obtained parameter differs from that of SGD. 

Based on our findings, we suggest using large step sizes for difficult datasets to learn so that SGD itself oscillates and performs somewhat poorly. Then, by averaging the parameters, we can expect a significant improvement in generalization.

One limitation is that to verify the improvement condition: Eq.~(\ref{eq:improvement_condition}), we imposed a sort of one-point strong convexity on $F$ in Proposition \ref{prop:sgd_upper_bound}, which is difficult to satisfy for neural networks perfectly. We believe Eq.~(\ref{eq:improvement_condition}) is met in a broader problem setting and will help us understand the actual behavior of averaged SGD in deep learning. Further exploration of this condition is left to future work.


\bibliographystyle{apalike}
\bibliography{ref}

\begin{thebibliography}{}

\bibitem[Ahn et~al., 2022]{pmlr-v162-ahn22a}
Ahn, K., Zhang, J., and Sra, S. (2022).
\newblock Understanding the unstable convergence of gradient descent.
\newblock In {\em Proceedings of the 39th International Conference on Machine
  Learning}, volume 162, pages 247--257.

\bibitem[Arora et~al., 2022]{pmlr-v162-arora22a}
Arora, S., Li, Z., and Panigrahi, A. (2022).
\newblock Understanding gradient descent on the edge of stability in deep
  learning.
\newblock In {\em Proceedings of the 39th International Conference on Machine
  Learning}, pages 948--1024.

\bibitem[Bach and Moulines, 2011]{bach2011non}
Bach, F. and Moulines, E. (2011).
\newblock Non-asymptotic analysis of stochastic approximation algorithms for
  machine learning.
\newblock In {\em Advances in Neural Information Processing Systems},
  volume~24, pages 451--459.

\bibitem[Chaudhari et~al., 2019]{chaudhari2019entropy}
Chaudhari, P., Choromanska, A., Soatto, S., LeCun, Y., Baldassi, C., Borgs, C.,
  Chayes, J., Sagun, L., and Zecchina, R. (2019).
\newblock Entropy-sgd: Biasing gradient descent into wide valleys.
\newblock {\em Journal of Statistical Mechanics: Theory and Experiment},
  2019(12):124018.

\bibitem[Cohen et~al., 2021]{cohen2020gradient}
Cohen, J., Kaur, S., Li, Y., Kolter, J.~Z., and Talwalkar, A. (2021).
\newblock Gradient descent on neural networks typically occurs at the edge of
  stability.
\newblock In {\em Proceedings of the 9th International Conference on Learning
  Representations}.

\bibitem[Damian et~al., 2021]{damian2021label}
Damian, A., Ma, T., and Lee, J.~D. (2021).
\newblock Label noise sgd provably prefers flat global minimizers.
\newblock {\em Advances in Neural Information Processing Systems},
  34:27449--27461.

\bibitem[DeVries and Taylor, 2017]{devries2017cutout}
DeVries, T. and Taylor, G.~W. (2017).
\newblock Improved regularization of convolutional neural networks with cutout.
\newblock {\em arXiv preprint arXiv:1708.04552}.

\bibitem[Dieuleveut et~al., 2020]{dieuleveut2020}
Dieuleveut, A., Durmus, A., and Bach, F. (2020).
\newblock Bridging the gap between constant step size stochastic gradient
  descent and markov chains.
\newblock {\em Annals of Statistics}, 48(3):1348--1382.

\bibitem[Dinh et~al., 2017]{dinh2017sharp}
Dinh, L., Pascanu, R., Bengio, S., and Bengio, Y. (2017).
\newblock Sharp minima can generalize for deep nets.
\newblock In {\em Proceedings of the 34th International Conference on Machine
  Learning}, pages 1019--1028.

\bibitem[Foret et~al., 2020]{foret2020sharpness}
Foret, P., Kleiner, A., Mobahi, H., and Neyshabur, B. (2020).
\newblock Sharpness-aware minimization for efficiently improving
  generalization.
\newblock In {\em Proceedings of the 8th International Conference on Learning
  Representations}.

\bibitem[Han et~al., 2017]{han2017deep}
Han, D., Kim, J., and Kim, J. (2017).
\newblock Deep pyramidal residual networks.
\newblock In {\em Proceedings of the IEEE conference on computer vision and
  pattern recognition}, pages 5927--5935.

\bibitem[Haruki et~al., 2019]{haruki2019gradient}
Haruki, K., Suzuki, T., Hamakawa, Y., Toda, T., Sakai, R., Ozawa, M., and
  Kimura, M. (2019).
\newblock Gradient noise convolution (gnc): Smoothing loss function for
  distributed large-batch sgd.
\newblock {\em arXiv preprint arXiv:1906.10822}.

\bibitem[He et~al., 2019]{he2019asymmetric}
He, H., Huang, G., and Yuan, Y. (2019).
\newblock Asymmetric valleys: beyond sharp and flat local minima.
\newblock In {\em Advances in Neural Information Processing Systems},
  volume~32, pages 2553--2564.

\bibitem[He et~al., 2016]{he2016deep}
He, K., Zhang, X., Ren, S., and Sun, J. (2016).
\newblock Deep residual learning for image recognition.
\newblock In {\em Proceedings of the IEEE conference on computer vision and
  pattern recognition}, pages 770--778.

\bibitem[Hinton and Van~Camp, 1993]{hinton1993keeping}
Hinton, G.~E. and Van~Camp, D. (1993).
\newblock Keeping the neural networks simple by minimizing the description
  length of the weights.
\newblock In {\em Proceedings of the sixth annual conference on Computational
  learning theory}, pages 5--13.

\bibitem[Hochreiter and Schmidhuber, 1997]{hochreiter1997flat}
Hochreiter, S. and Schmidhuber, J. (1997).
\newblock Flat minima.
\newblock {\em Neural computation}, 9(1):1--42.

\bibitem[Izmailov et~al., 2018]{izmailov2018averaging}
Izmailov, P., Podoprikhin, D., Garipov, T., Vetrov, D., and Wilson, A.~G.
  (2018).
\newblock Averaging weights leads to wider optima and better generalization.
\newblock In {\em 34th Conference on Uncertainty in Artificial Intelligence},
  pages 876--885.

\bibitem[Jiang et~al., 2020]{jiang2020fantastic}
Jiang, Y., Neyshabur, B., Mobahi, H., Krishnan, D., and Bengio, S. (2020).
\newblock Fantastic generalization measures and where to find them.
\newblock In {\em Proceedings of the 8th International Conference on Learning
  Representations}.

\bibitem[Kaddour et~al., 2022]{kaddour2022}
Kaddour, J., Liu, L., Silva, R., and Kusner, M. (2022).
\newblock When do flat minima optimizers work?
\newblock {\em Advances in Neural Information Processing Systems}.

\bibitem[Keskar et~al., 2017]{keskar2017large}
Keskar, N.~S., Nocedal, J., Tang, P. T.~P., Mudigere, D., and Smelyanskiy, M.
  (2017).
\newblock On large-batch training for deep learning: Generalization gap and
  sharp minima.
\newblock In {\em Proceedings of the 5th International Conference on Learning
  Representations}.

\bibitem[Keskar and Socher, 2017]{keskar2017improving}
Keskar, N.~S. and Socher, R. (2017).
\newblock Improving generalization performance by switching from adam to sgd.
\newblock {\em arXiv preprint arXiv:1712.07628}.

\bibitem[Kleinberg et~al., 2018]{kleinberg2018alternative}
Kleinberg, B., Li, Y., and Yuan, Y. (2018).
\newblock An alternative view: When does sgd escape local minima?
\newblock In {\em Proceedings of the 35th International Conference on Machine
  Learning}, pages 2698--2707.

\bibitem[Lacoste-Julien et~al., 2012]{lacoste2012simpler}
Lacoste-Julien, S., Schmidt, M., and Bach, F. (2012).
\newblock A simpler approach to obtaining an o (1/t) convergence rate for the
  projected stochastic subgradient method.
\newblock {\em arXiv preprint arXiv:1212.2002}.

\bibitem[Lewkowycz et~al., 2020]{lewkowycz2020large}
Lewkowycz, A., Bahri, Y., Dyer, E., Sohl-Dickstein, J., and Gur-Ari, G. (2020).
\newblock The large learning rate phase of deep learning: the catapult
  mechanism.
\newblock {\em arXiv preprint arXiv:2003.02218}.

\bibitem[Lin et~al., 2020]{lin2020extrapolation}
Lin, T., Kong, L., Stich, S., and Jaggi, M. (2020).
\newblock Extrapolation for large-batch training in deep learning.
\newblock In {\em Proceedings of the 37th International Conference on Machine
  Learning}, pages 6094--6104.

\bibitem[Luo et~al., 2019]{luo2019adaptive}
Luo, L., Xiong, Y., Liu, Y., and Sun, X. (2019).
\newblock Adaptive gradient methods with dynamic bound of learning rate.
\newblock In {\em Proceedings of the 7th International Conference on Learning
  Representations}.

\bibitem[McAllester, 1998]{mcallester1998some}
McAllester, D.~A. (1998).
\newblock Some pac-bayesian theorems.
\newblock In {\em Proceedings of the eleventh annual conference on
  Computational learning theory}, pages 230--234.

\bibitem[McAllester, 1999]{mcallester1999pac}
McAllester, D.~A. (1999).
\newblock Pac-bayesian model averaging.
\newblock In {\em Proceedings of the twelfth annual conference on Computational
  learning theory}, pages 164--170.

\bibitem[Nemirovski et~al., 2009]{nemirovski2009robust}
Nemirovski, A.~S., Juditsky, A., Lan, G., and Shapiro, A. (2009).
\newblock Robust stochastic approximation approach to stochastic programming.
\newblock {\em SIAM Journal on Optimization}, 19(4):1574--1609.

\bibitem[Nesterov, 2004]{nes2004}
Nesterov, Y. (2004).
\newblock {\em Introductory Lectures on Convex Optimization: A Basic Course}.
\newblock Kluwer Academic Publishers.

\bibitem[Neyshabur et~al., 2017]{neyshabur2017exploring}
Neyshabur, B., Bhojanapalli, S., McAllester, D., and Srebro, N. (2017).
\newblock Exploring generalization in deep learning.
\newblock {\em Advances in neural information processing systems}, 30.

\bibitem[Orvieto et~al., 2022]{orvieto2022explicit}
Orvieto, A., Raj, A., Kersting, H., and Bach, F. (2022).
\newblock Explicit regularization in overparametrized models via noise
  injection.
\newblock {\em arXiv preprint arXiv:2206.04613}.

\bibitem[Polyak and Juditsky, 1992]{polyak1992acceleration}
Polyak, B.~T. and Juditsky, A.~B. (1992).
\newblock Acceleration of stochastic approximation by averaging.
\newblock {\em SIAM Journal on Control and Optimization}, 30(4):838--855.

\bibitem[Rakhlin et~al., 2012]{rakhlin2012making}
Rakhlin, A., Shamir, O., and Sridharan, K. (2012).
\newblock Making gradient descent optimal for strongly convex stochastic
  optimization.
\newblock In {\em Proceedings of International Conference on Machine Learning
  29}, pages 1571--1578.

\bibitem[Robbins and Monro, 1951]{robbins1951stochastic}
Robbins, H. and Monro, S. (1951).
\newblock A stochastic approximation method.
\newblock {\em The annals of mathematical statistics}, pages 400--407.

\bibitem[Ruppert, 1988]{ruppert1988efficient}
Ruppert, D. (1988).
\newblock Efficient estimations from a slowly convergent {R}obbins-{M}onro
  process.
\newblock Technical report, Cornell University Operations Research and
  Industrial Engineering.

\bibitem[Smith et~al., 2017]{smith2017don}
Smith, S.~L., Kindermans, P.-J., Ying, C., and Le, Q.~V. (2017).
\newblock Don't decay the learning rate, increase the batch size.
\newblock {\em arXiv preprint arXiv:1711.00489}.

\bibitem[Wen et~al., 2018]{wen2018smoothout}
Wen, W., Wang, Y., Yan, F., Xu, C., Wu, C., Chen, Y., and Li, H. (2018).
\newblock Smoothout: Smoothing out sharp minima to improve generalization in
  deep learning.
\newblock {\em arXiv preprint arXiv:1805.07898}.

\bibitem[Wilson et~al., 2017]{wilson2017marginal}
Wilson, A.~C., Roelofs, R., Stern, M., Srebro, N., and Recht, B. (2017).
\newblock The marginal value of adaptive gradient methods in machine learning.
\newblock {\em Advances in neural information processing systems}, 30.

\bibitem[Yu et~al., 2020]{yu2020analysis}
Yu, L., Balasubramanian, K., Volgushev, S., and Erdogdu, M.~A. (2020).
\newblock An analysis of constant step size sgd in the non-convex regime:
  Asymptotic normality and bias.
\newblock {\em arXiv preprint arXiv:2006.07904}.

\bibitem[Zagoruyko and Komodakis, 2016]{zagoruyko2016wide}
Zagoruyko, S. and Komodakis, N. (2016).
\newblock Wide residual networks.
\newblock In {\em British Machine Vision Conference 2016}. British Machine
  Vision Association.

\bibitem[Zhou et~al., 2020]{zhou2020towards}
Zhou, P., Feng, J., Ma, C., Xiong, C., Hoi, S. C.~H., et~al. (2020).
\newblock Towards theoretically understanding why sgd generalizes better than
  adam in deep learning.
\newblock {\em Advances in Neural Information Processing Systems},
  33:21285--21296.

\end{thebibliography}

\clearpage
\onecolumn

{
\newgeometry{top=1in, bottom=1in,left=1.in,right=1.in}   
\renewcommand{\contentsname}{Table of Contents}
\tableofcontents

\newpage

\allowdisplaybreaks

\linewidth\hsize
{\centering \Large\bfseries Appendix: Why is parameter averaging beneficial in SGD? \\An objective smoothing perspective \par}
\section{Proofs}
\subsection{Alternative Stochastic Gradient Descent}\label{subsec:isgd}
Let denote by $\varphi: \bR^d \rightarrow \bR^d$ a change of variables from $w$ to $v$ introduced in Section \ref{subsec:alt_view}, i.e., $v = \varphi(w) = w - \eta \nabla f(w)$.
\begin{lemma}
Under Assumption {\bf (A1)} and $\eta \leq \frac{1}{2L}$, the function $\varphi$ is injective and invertible, and its inverse $\varphi^{-1}$ defined on on $\im\varphi$ is differentiable.
\end{lemma}

\begin{proof}
For $w, w' \in \bR^d$, we suppose $\varphi(w)=\varphi(w')$. Then, it holds that 
\[ \| w - w' \| = \eta \| \nabla f(w) - \nabla f(w') \| \leq \eta L \| w - w '\| \leq \frac{1}{2}\|w-w'\|, \]
where we used $L$-Lipschitz continuity of $\nabla f$ due to {\bf (A1)}. Therefore, we see $w = w'$ and $\varphi$ is an injection.
Moreover, since $J_\varphi(w) = I - \eta \nabla^2 f(w) \succeq (1-\eta L) I \succeq \frac{1}{2}I$. Thus, $\varphi$ is invertible and $\varphi^{-1}$, which is defined on $\im\varphi$, is differentiable because of the injectivity and the inverse map theorem.
\end{proof}

Using $\varphi$, we see $\epsilon'(v) = \epsilon(\varphi^{-1}(v))$ for $v \in \im \varphi$.
Let $(\Omega, \cF, P)$ be a probability space such that $\epsilon'(v)$ can be represented as a measurable map $z \in \Omega \mapsto \epsilon'(v,z)$. Note that we use $\epsilon'(v)$ and $\epsilon'(v,z)$ depending on the situation.
For a function $g: \bR^d \rightarrow \bR^d$, we denote by $J_g (w)$ Jacobian of $g$, i.e., $J_g(w) = ( \partial g_i(w) /\partial w_j )_{i,j=1}^d$.

\begin{lemma}\label{lemma:biased_gradient}
Under Assumptions {\bf (A1)} and  {\bf (A2)}, we get for any $v \in \im \varphi \subset \bR^d$, 
\begin{equation*}
    \nabla F(v) = \bE[ \nabla f( v - \eta \epsilon'(v))] - \eta \int J_{\epsilon'(\cdot,z)}^\top(v)\nabla f( v - \eta \epsilon'(v,z)) \rmd P(z).
\end{equation*}
Moreover, if $\eta \leq \frac{1}{2L}$, then 
\begin{equation*}
\| \nabla F(v) - \bE[ \nabla f( v - \eta \epsilon'(v))] \| \leq  2 \eta \sigma_2 \sqrt{\bE\left[ \| \nabla f( v - \eta \epsilon'(v)) \|^2 \right]}.    
\end{equation*}
\end{lemma}

\begin{proof}
The first equality of the statement can be confirmed by the direct calculation as follows:
\begin{align*}
    \nabla F(v) 
    &= \nabla \bE[ f( v - \eta \epsilon'(v))] \\
    &= \int \nabla (f( v - \eta \epsilon'(v,z))) \rmd P(z) \\
    &= \int (I - \eta J_{\epsilon'(\cdot,z)}^\top(v))\nabla f( v - \eta \epsilon'(v,z)) \rmd P(z) \\
    &= \bE[ \nabla f( v - \eta \epsilon'(v)) ] - \eta \int J_{\epsilon'(\cdot,z)}^\top(v)\nabla f( v - \eta \epsilon'(v,z)) \rmd P(z).
\end{align*}

Next, we evaluate the last term below.
By the chain rule and inverse map theory, 
\[ J_{\epsilon'(\cdot,z)}(v) = J_{\epsilon(\phi^{-1}(\cdot),z)}(v) = J_{\epsilon(\cdot,z)}( \phi^{-1}(v)) J_{\phi^{-1}}(v) = J_{\epsilon(\cdot,z)}( \varphi^{-1}(v)) J_{\varphi}^{-1}(\varphi^{-1}(v)).\]

Note that from assumption for any $w \in \bR^d$, $J_{\varphi}(w) = I - \eta \nabla^2 f(w) \succeq (1 - \eta L )I \succeq \frac{1}{2}I$.
Hence, 
\[ \| J_{\epsilon'(\cdot,z)}^\top (v) \|_2 
\leq  \| J_{\varphi}^{-1}(\varphi^{-1}(v))\|_2 \| J_{\epsilon(\cdot,z)}^\top ( \varphi^{-1}(v)) \|_2 
\leq 2 \sigma_2. \]

Finally, we get
\begin{align*}
    &\hspace{-10mm}\left\| \int J_{\epsilon'(\cdot,z)}^\top(v) \nabla f( v - \eta \epsilon'(v,z)) \rmd P(z) \right\| \\ 
    &\leq \sqrt{\int \| J_{\epsilon'(\cdot,z)}^\top(v) \|_2^2 \| \nabla f( v - \eta \epsilon'(v,z)) \|^2 \rmd P(z)} \\
    &\leq 2 \sigma_2 \sqrt{\int  \| \nabla f( v - \eta \epsilon'(v,z)) \|^2 \rmd P(z)} \\
    &\leq 2 \sigma_2 \sqrt{\bE\left[ \| \nabla f( v - \eta \epsilon'(v)) \|^2 \right]}.
\end{align*}
This finishes the proof.
\end{proof}

The following proposition is the restatement of the well-known convergence result to a stationary point using the coordinate $v$.
\begin{proposition}\label{prop:convergence_stationary_point}
Under Assumptions {\bf (A1)}, {\bf (A2)}, and $\eta \leq \frac{1}{2L}$, we get
\begin{equation}
    \sum_{t=0}^{T} \bE\left[\|\nabla f( v_t - \eta \epsilon_{t+1}'(v_t))\|^2\right] 
    \leq \frac{4}{3\eta}\bE[f(w_0)] + \frac{2}{3}\eta \sigma_1^2 L (T+2).
\end{equation}
\end{proposition}
\begin{proof}
It is known that {\bf (A1}) derives the following \citep{nes2004}: for any $w, w'\in \bR^d$, 
\begin{equation}\label{eq:lipschitz_smoothness}
    f(w') \leq f(w) + \nabla f(w)^\top (w'-w ) + \frac{L}{2}\|w' - w\|^2.    
\end{equation}
Substituting the update Eq.~(\ref{eq:sgd}) into this inequality with $w'=w_{t+1}$ and $w=w_t$, and taking the conditional expectation $\bE[\cdot|\cF_t]$, we get
\begin{align*}
    \bE[f(w_{t+1}) | \cF_t] 
    &\leq f(w_t) -\eta \|\nabla f(w_t)\|^2 + \frac{\eta^2 L}{2} \bE\left[\| \nabla f(w_t) + \epsilon_{t+1}(w_t) \|^2 | \cF_t \right] \\
    &= f(w_t) -\eta \left(1 - \frac{\eta L }{2}\right) \|\nabla f(w_t)\|^2 + \frac{\eta^2 L}{2} \bE\left[\| \epsilon_{t+1}(w_t) \|^2 | \cF_t \right] \\
    &\leq f(w_t) - \frac{3\eta}{4} \|\nabla f(w_t)\|^2 + \frac{\eta^2 \sigma_1^2 L}{2}. 
\end{align*}
Thus, we have $\bE[f(w_{t+1})] \leq \bE[f(w_t)] - \frac{3\eta}{4} \bE[\|\nabla f(w_t)\|^2] + \frac{\eta^2 \sigma_1^2 L}{2}$.
By summing up this inequality, we get
\[ \sum_{t=0}^{T+1} \bE[\|\nabla f(w_t)\|^2] \leq \frac{4}{3\eta}\bE[f(w_0)] + \frac{2}{3}\eta \sigma_1^2 L (T+2), \]
where we used the nonnegativity of $f$.
By dropping the term with $t=0$ of the sum in the left hand side and using $w_{t+1} = v_t - \eta \epsilon_{t+1}'(v_t)$, we finally get
\[ \sum_{t=0}^{T} \bE[\|\nabla f( v_t - \eta \epsilon_{t+1}'(v_t))\|^2] \leq \frac{4}{3\eta}\bE[f(w_0)] + \frac{2}{3}\eta \sigma_1^2 L (T+2). \]
\end{proof}

\subsection{Proof of Theorem \ref{thm:asgd}}\label{subsec:proof_asgd}
We give technical lemmas for proving Theorem \ref{thm:asgd}.
\begin{lemma}\label{lem:sgd_lem}
Under Assumptions {\bf(A1)} and {\bf(A2)}, run the stochastic gradient descent with $T$-iterations with the step size $\eta \leq \frac{1}{2L}$, then the alternative SGD satisfies the following inequality:
\begin{align*}
& \bE\left[  \left\| \sum_{t=0}^T \nabla f(v_t - \eta \epsilon'_{t+1}(v_t) ) \right\|  \right] 
\leq \sigma_1 \sqrt{T} + \frac{1}{\eta} \sqrt{ 2 \sum_{t=0}^{T+1}\bE\left[ \left\| v_t - v_* \right\|^2 \right] }, \\
&\bE \left[ \left\| \sum_{t=0}^T  \left(\nabla F(v_t) - \bE\left[ \nabla f(v_t - \eta \epsilon'_{t+1}(v_t)) | \cF_t \right] \right) \right\| \right] 
\leq 2 \sigma_2 \eta^{\frac{1}{2}} O(T^{\frac{1}{2}}) + 2\sigma_1 \sigma_2 \eta^{\frac{3}{2}} \sqrt{\frac{2}{3} L (T+1)(T+2)  }.
\end{align*}
\end{lemma}

\begin{proof}[Proof of Lemma \ref{lem:sgd_lem}]
By the simple calculation, we get
\begin{align*}
    \bE\left[ \left\|\sum_{t=0}^T \nabla f(v_t - \eta \epsilon'_{t+1}(v_t) ) \right\| \right] 
    &\leq  \bE\left[ \left\|\sum_{t=0}^T \left( \nabla f(v_t - \eta \epsilon'_{t+1}(v_t) ) + \epsilon'_{t+1}(v_t) \right) \right\|\right]  
    + \bE\left[ \left\|\sum_{t=0}^T \epsilon'_{t+1}(v_t) \right\|\right] \\
    &\leq \frac{1}{\eta}   \bE\left[\left\| v_0 - v_{T+1}  \right\| \right]
    + \bE\left[ \left\|\sum_{t=0}^T \epsilon'_{t+1}(v_t) \right\|\right].
\end{align*}
Each term in the right-hand side can be evaluated as follows:
\begin{align*}
    \frac{1}{\eta} \bE\left[\left\| v_0 - v_{T+1}  \right\| \right] 
    &\leq \frac{1}{\eta}  \sqrt{\bE\left[ \left\| v_0 - v_{T+1}  \right\|^2 \right] } \\
    &\leq \frac{1}{\eta} \sqrt{ 2\bE\left[ \left\| v_0 - v_* \right\|^2 + \left\| v_{T+1} - v_* \right\|^2 \right] } \\
    &\leq \frac{1}{\eta} \sqrt{ 2 \sum_{t=0}^{T+1}\bE\left[ \left\| v_t - v_* \right\|^2 \right] }.
\end{align*}
\begin{align*}
    \bE\left[ \left\|\sum_{t=0}^T \epsilon'_{t+1}(v_t) \right\|\right] 
    &\leq \sqrt{\bE\left[ \left\|\sum_{t=0}^T \epsilon'_{t+1}(v_t) \right\|^2\right] } \\
    &= \sqrt{\bE\left[ \sum_{t=0}^T \left\|\epsilon'_{t+1}(v_t) \right\|^2\right] } \\
    &\leq \sigma_1 \sqrt{T},
\end{align*}
where we used that for $s < t$, $\bE[ \epsilon'_{s+1}(v_s)^\top \epsilon'_{t+1}(v_t) ] = \bE[ \epsilon'_{s+1}(v_s)^\top \bE[ \epsilon'_{t+1}(v_t) | \cF_t]] = 0$.

Next, we show the second statement by using Lemma \ref{lemma:biased_gradient} and Proposition \ref{prop:convergence_stationary_point} as follows:
\begin{align*}
&\hspace{-10mm}\bE \left[ \left\| \sum_{t=0}^T  \left(\nabla F(v_t) - \bE\left[ \nabla f(v_t - \eta \epsilon'_{t+1}(v_t)) | \cF_t \right] \right) \right\| \right] \\
&\leq \bE \left[ \sum_{t=0}^T \left\|  \nabla F(v_t) - \bE\left[\nabla f(v_t - \eta \epsilon'_{t+1}(v_t)) |\cF_t \right] \right\| \right] \\
&\leq \bE \left[ \sum_{t=0}^T 2 \eta \sigma_2  \sqrt{ \bE \left[ \| \nabla f( v_t - \eta \epsilon_{t+1}'(v_t)) \|^2 | \cF_t \right] } \right] \\
&\leq 2 \eta \sigma_2  \sum_{t=0}^T \sqrt{ \bE \left[ \| \nabla f( v_t - \eta \epsilon_{t+1}'(v_t)) \|^2 \right] } \\
&\leq 2 \eta \sigma_2  \sqrt{(T+1)\sum_{t=0}^T \bE \left[ \| \nabla f( v_t - \eta \epsilon_{t+1}'(v_t)) \|^2 \right] } \\
&\leq 2 \eta \sigma_2  \sqrt{\frac{1}{\eta}O(T) + \frac{2}{3}\eta \sigma_1^2 L (T+1)(T+2)  } \\
&\leq 2 \sigma_2  \eta^{\frac{1}{2}}O(T^{\frac{1}{2}}) + 2 \sigma_1\sigma_2 \eta^{\frac{3}{2}}\sqrt{\frac{2}{3} L (T+1)(T+2)  }.
\end{align*}
\end{proof}

\begin{lemma}\label{lemma:sgd_appendix_2}
Under Assumptions {\bf(A1)} and {\bf(A2)}, run the stochastic gradient descent with $T$-iterations with the step size $\eta \leq \frac{1}{2L}$, then the alternative SGD satisfies the following inequality:
\begin{align*}
    \frac{1}{T+1}\bE\left[ \left\| \sum_{t=0}^T \nabla F(v_t) \right\| \right] 
    \leq O(T^{-\frac{1}{2}}) 
    + \frac{1}{\eta (T+1)} \sqrt{ 2 \sum_{t=0}^{T+1}\bE\left[ \left\| v_t - v_* \right\|^2 \right] }
    + \frac{4}{\sqrt{3}}\sigma_1 \sigma_2 \eta^{\frac{3}{2}}L^{\frac{1}{2}}.
\end{align*}
\end{lemma}

\begin{proof}[Proof of Lemma \ref{lemma:sgd_appendix_2}]
We consider the following decomposition:
\begin{align*}
    \frac{1}{T+1} \bE\left[ \left\| \sum_{t=0}^T \nabla F(v_t) \right\| \right] 
    &\leq \frac{1}{T+1} \bE \left[  \left\| \sum_{t=0}^T  \left(\nabla F(v_t) - \bE\left[ \nabla f(v_t - \eta \epsilon'_{t+1}(v_t)) | \cF_t \right] \right) \right\| \right]  \\
    &+ \frac{1}{T+1} \bE \left[  \left\| \sum_{t=0}^T \left( \bE\left[ \nabla f(v_t - \eta \epsilon'_{t+1}(v_t)) | \cF_t \right] - \nabla f(v_t - \eta \epsilon'_{t+1}(v_t)) \right)\right\| \right]  \\    
    &+ \frac{1}{T+1} \bE\left[ \left\| \sum_{t=0}^T \nabla f(v_t - \eta \epsilon'_{t+1}(v_t) ) \right\| \right]. 
\end{align*}
The first and last terms of the right-hand side can be upper bounded by Lemma \ref{lem:sgd_lem}.
The second term can be evaluated as follows: 
\begin{align*}
    &\bE \left[  \left\| \sum_{t=0}^T \left( \bE\left[ \nabla f(v_t - \eta \epsilon'_{t+1}(v_t)) | \cF_t \right]  - \nabla f(v_t - \eta \epsilon'_{t+1}(v_t)) \right)\right\| \right] \\
    &\leq \sqrt{\bE \left[  \left\| \sum_{t=0}^T \left( \bE\left[ \nabla f(v_t - \eta \epsilon'_{t+1}(v_t)) | \cF_t \right]  - \nabla f(v_t - \eta \epsilon'_{t+1}(v_t)) \right)\right\|^2 \right]} \\
    &= \sqrt{\bE \left[ \sum_{t=0}^T \left\|  \left( \bE\left[ \nabla f(v_t - \eta \epsilon'_{t+1}(v_t)) | \cF_t \right]  - \nabla f(v_t - \eta \epsilon'_{t+1}(v_t)) \right)\right\|^2 \right]} \\
    &\leq \sqrt{\bE \left[ 4 \eta^2 L^2 \sum_{t=0}^T  \bE\left[ \left\|\epsilon'_{t+1}(v_t)) \right\|^2 | \cF_t \right]  \right]}  \\
    &\leq 2 \eta L \sigma_1 \sqrt{T},
\end{align*}
where we used that for $X_t :=  \bE\left[ \nabla f(v_t - \eta \epsilon'_{t+1}(v_t)) | \cF_t \right]  - \nabla f(v_t - \eta \epsilon'_{t+1}(v_t))$ and $s<t$, $\bE[ X_s X_t] = \bE[ X_s \bE[X_t | \cF_t] ] = 0$.
Therefore, we eventually obtain the statement of the proposition.
\end{proof}

We here prove Theorem \ref{thm:asgd} which is restated below.
\begin{theorem}\label{thm:asgd_appendix}
Under Assumptions {\bf (A1)}--{\bf (A4)}, run the averaged SGD for $T$-iterations with the step size $\eta \leq \frac{1}{2L}$, then the average $\overline{v}_T$ satisfies the following inequality:
\begin{equation*}
    \bE[\| \overline{v}_T - v_* \|]
    \leq O\left( T^{-\frac{1}{2}} \right) 
    + \frac{4\sigma_1 \sigma_2 \eta^{\frac{3}{2}} L^{\frac{1}{2}}}{\sqrt{3} \mu} 
    + \frac{D_T}{\eta \mu}\sqrt{\frac{2}{T+1}}
    + \frac{\gamma + M D_T^2}{\mu},
\end{equation*}
where we set $D_T := \sqrt{\frac{1}{T+1}\bE\left[\sum_{t=0}^T \left\| v_t - v_* \right\|^2 \right]}$.
\end{theorem}

\begin{proof}
We define $R(v) = \nabla F(v) - \nabla^2 F(v_*) (v - v_*)$. Then, by {\bf (A4)}, we see $\| R(v) \| \leq \gamma + M \|v-v_*\|^2$.
By taking average of $R(v_t)$ over $t\in \{0,1,\ldots,T\}$ and rearranging terms, we get
\[ \nabla^2 F(v_*) (\overline{v}_T - v_*) = \frac{1}{T+1}\sum_{t=0}^T \nabla F(v_t) - \frac{1}{T+1}\sum_{t=0}^T R(v_t). \]

Using the positivity of $\nabla^2 F(v_*)$ and taking the expectation, we get
\begin{align*}
 \mu \bE[ \| \overline{v}_T - v_* \| ]
 &\leq  \bE[ \| \nabla^2 F(v_*) ( \overline{v}_T - v_*)   \| ]\\
 &\leq \frac{1}{T+1} \bE\left[ \left\| \sum_{t=0}^T \nabla F(v_t)  \right\| \right] 
 + \frac{1}{T+1} \bE\left[ \left\|  \sum_{t=0}^T R(v_t) \right\| \right] \\
 &\leq \frac{1}{T+1} \bE\left[ \left\|  \sum_{t=0}^T \nabla F(v_t) \right\| \right]
 + \gamma
 + \frac{M}{T+1} \bE\left[   \sum_{t=0}^T \left\| v_t - v_* \right\|^2 \right].
\end{align*}

The first term of the right-hand side can be bounded by Lemma \ref{lemma:sgd_appendix_2}. Thus, we finally get 
\begin{equation*}
    \mu \bE[ \|\overline{v}_T - v_* \|]
    \leq O\left( T^{-\frac{1}{2}} \right) 
    + \frac{4\sigma_1 \sigma_2 \eta^{\frac{3}{2}} L^{\frac{1}{2}}}{\sqrt{3}} 
    + \frac{1}{\eta (T+1)} \sqrt{ 2 \sum_{t=0}^{T+1}\bE\left[ \left\| v_t - v_* \right\|^2 \right] }
    + \gamma
    + \frac{M}{T+1} \bE\left[   \sum_{t=0}^T \left\| v_t - v_* \right\|^2 \right].
\end{equation*}
Moreover, Jensen's inequality yields $\bE[ \| \overline{v}_T - v_* \|] \leq \sqrt{\bE[ \| \overline{v}_T - v_* \|^2]}  \leq D_T$.
This concludes the proof.
\end{proof}

\subsection{Proof of Propositions \ref{prop:sgd_upper_bound} and \ref{prop:sgd_lower_bound}}
We prove Proposition \ref{prop:sgd_upper_bound}, which is restated below.
\begin{proposition}\label{prop:sgd_appendix}
Under Assumptions {\bf (A1)}, {\bf (A2)}, and {\bf (A5)}, run SGD for $T$-iterations with the step size $\eta \leq \frac{1}{2L}$, then we get
\begin{align*}
    D_T^2 \leq O\left(T^{-1}\right) 
    + \frac{2\eta\sigma_1^2}{c} 
    + \frac{8\eta^2 \sigma_1^2 L}{3c} \left(1 + \frac{2 \eta \sigma_2^2}{c} \right). 
\end{align*}
\end{proposition}
\begin{proof}
To evaluate $\|v_{t+1} - v_*\|^2$ for the alternative SGD (\ref{eq:implicit_sgd}), we first give several bounds as follows.
By Assumption {\bf (A3)}, Young's inequality, and Lemma \ref{lemma:biased_gradient}, we get
\begin{align*}
    &- 2 (v_t - v_*)^\top \bE[ \nabla f( v_t - \eta \epsilon'_{t+1}(v_t) ) | \cF_t] \\
    &= - 2 (v_t - v_*)^\top \nabla F(v_t) + 2 (v_t - v_*)^\top (\nabla F(v_t) - \bE[ \nabla f( v_t - \eta \epsilon'_{t+1}(v_t) ) | \cF_t] ) \\
    &\leq -2c \|v_t - v_*\|^2 + c\|v_t-v_*\|^2 + \frac{1}{c} \| \nabla F(v_t) - \bE[ \nabla f( v_t - \eta \epsilon'_{t+1}(v_t) ) | \cF_t] \|^2 \\
    &\leq -c \|v_t - v_*\|^2 + \frac{4 \eta^2 \sigma_2^2}{c} \bE\left[ \| \nabla f( v_t - \eta \epsilon_{t+1}'(v_t)) \|^2 | \cF_t \right].
\end{align*}
By Assumption {\bf (A2)} and Young's inequality again, we get
\begin{align*}
    &\bE[ \| \epsilon'_{t+1}(v_t) + \nabla f( v_t - \eta \epsilon'_{t+1}(v_t) ) \|^2 | \cF_t] \\
    &\leq 2 \bE[ \| \epsilon'_{t+1}(v_t) \|^2 | \cF_t ] + 2\bE[ \| \nabla f( v_t - \eta \epsilon'_{t+1}(v_t) ) \|^2 | \cF_t] \\
    &\leq 2 \sigma_1^2 + 2 \bE[ \| \nabla f( v_t - \eta \epsilon'_{t+1}(v_t) ) \|^2 | \cF_t].
\end{align*}
Combining the above two inequalities, we get
\begin{align*}
    \bE[ \| v_{t+1} - v_* \|^2 | \cF_t]
    &= \bE[ \|  v_t - \eta \epsilon'_{t+1}(v_t) - \eta \nabla f( v_t - \eta \epsilon'_{t+1}(v_t) ) - v_* \|^2 | \cF_t] \\
    &= \|  v_t - v_* \|^2 
    - 2\eta (v_t - v_*)^\top \bE[ \nabla f( v_t - \eta \epsilon'_{t+1}(v_t) ) | \cF_t] \\
    &+ \eta^2 \bE[ \| \epsilon'_{t+1}(v_t) + \nabla f( v_t - \eta \epsilon'_{t+1}(v_t) ) \|^2 | \cF_t] \\
    &\leq (1-c\eta)\|  v_t - v_* \|^2  + 2 \eta^2\sigma_1^2 \\
    &+ 2\eta^2 \left(1 + \frac{2 \eta \sigma_2^2}{c} \right) \bE\left[ \| \nabla f( v_t - \eta \epsilon_{t+1}'(v_t)) \|^2 | \cF_t \right].
\end{align*}
Taking the expectation regarding all histories and summing up over $t = 0, 1, \ldots, T$, we get
\begin{align*}
    c\eta\sum_{t=0}^T \bE[\|  v_t - v_* \|^2]
    &\leq \bE[\|  v_0 - v_* \|^2]  - \bE[ \| v_{T+1} - v_* \|^2 ] + 2 \eta^2\sigma_1^2 (T+1)\\
    &+ 2\eta^2 \left(1 + \frac{2 \eta \sigma_2^2}{c} \right) \sum_{t=0}^T \bE\left[ \| \nabla f( v_t - \eta \epsilon_{t+1}'(v_t)) \|^2 \right]\\
    &\leq \bE[\|  v_0 - v_* \|^2]  - \bE[ \| v_{T+1} - v_* \|^2 ] + 2 \eta^2\sigma_1^2 (T+1)\\
    &+ \frac{8}{3}\eta \left(1 + \frac{2 \eta \sigma_2^2}{c} \right) \bE[f(w_0)]
    + \frac{4}{3}\eta^3 \left(1 + \frac{2 \eta \sigma_2^2}{c} \right) \sigma_1^2 L (T+2),    
\end{align*}
where we used Proposition \ref{prop:convergence_stationary_point}.
Therefore, we conclude
\begin{align*}
    \frac{1}{T+1}\sum_{t=0}^T \bE[\|  v_t - v_* \|^2]
    &\leq \frac{1}{c\eta(T+1)}\bE[\|  v_0 - v_* \|^2]  + \frac{8}{3c(T+1)} \left(1 + \frac{2 \eta \sigma_2^2}{c} \right) \bE[f(w_0)] \\
    &+ \frac{2\eta\sigma_1^2}{c} + \frac{8\eta^2}{3c} \left(1 + \frac{2 \eta \sigma_2^2}{c} \right) \sigma_1^2 L \\
    &= O\left(T^{-1}\right) + \frac{2\eta\sigma_1^2}{c} + \frac{8\eta^2}{3c} \left(1 + \frac{2 \eta \sigma_2^2}{c} \right) \sigma_1^2 L.
\end{align*}
\end{proof}

We prove Proposition \ref{prop:sgd_lower_bound}, which is restated below.
\begin{proposition}
Suppose that Assumptions {\bf(A1)} and {\bf(A2)} hold, and that there exists $\sigma_3>0$ such that for any $w\in\bR^d$, $\bE[\|\epsilon_{t+1}(w)\|^2] \geq \sigma_3^2$.
Running SGD with $\eta \leq \frac{3\sigma_3^2}{32\sigma_1^2 L}$, we get
\[ \frac{\eta^2\sigma_3^2}{8} \leq D_T^2 + O(T^{-1}). \]
\end{proposition}
\begin{proof}
Let us evaluate the term $\bE[\| w_{t} - v_*\|^2 ]$ as follows.
On one hand, by using the update rule of SGD, for $t\geq 1$,
\begin{align*}
    \bE\left[\| w_{t} - v_*\|^2 \right]
    &= \bE\left[\| w_{t-1} - \eta \nabla f(w_{t-1}) - \eta \epsilon_{t}(w_{t-1}) - v_*\|^2 \right] \\
    &= \bE\left[\| w_{t-1} - \eta \nabla f(w_{t-1}) - v_* \|^2 
    - 2\eta \epsilon_{t}(w_{t-1})^\top (w_{t-1} - \eta \nabla f(w_{t-1}) - v_*)
    + \eta^2 \| \epsilon_{t}(w_{t-1}) \|^2 \right] \\
    &= \bE\left[\| w_{t-1} - \eta \nabla f(w_{t-1}) - v_* \|^2 
    + \eta^2 \| \epsilon_{t}(w_{t-1}) \|^2 \right] \\
    &\geq \eta^2 \sigma_3^2,
\end{align*}
where we used $\bE[\epsilon_{t}(w_{t-1})^\top (w_{t-1} - \eta \nabla f(w_{t-1}) - v_*)] = \bE[ \bE[\epsilon_{t}(w_{t-1}) | \cF_{t-1} ]^\top (w_{t-1} - \eta \nabla f(w_{t-1}) - v_*)] = 0$.
On the other hand, by using $v_t = w_t - \eta \nabla f(w_t)$,
\begin{align*}
    \bE\left[\| w_{t} - v_*\|^2 \right]
    &= \bE\left[\| v_{t} + \eta \nabla f(w_{t}) - v_*\|^2 \right] \\
    &\leq 2\bE\left[\| v_{t} - v_* \|^2 + \eta^2 \| \nabla f(w_{t}) \|^2 \right].
\end{align*}
Hence, by taking the sum over $t \in \{0,1,\ldots,T\}$, we have 
\begin{align*}
    \frac{1}{2}\|w_0 - v_*\|^2 + \frac{1}{2}\eta^2 \sigma_3^2 T
    &\leq \sum_{t=0}^T \bE\left[\| v_{t} - v_* \|^2 \right] 
    + \eta^2 \sum_{t=0}^T \bE\left[ \| \nabla f(w_{t}) \|^2 \right] \\
    &\leq \sum_{t=0}^T \bE\left[\| v_{t} - v_* \|^2 \right] 
    + \eta^2 \sum_{t=0}^T \bE\left[ \| \nabla f(w_{t+1}) \|^2 \right]
    + \eta^2 \| \nabla f(w_0)\|^2. 
\end{align*}
Since $w_{t+1} = v_t - \eta \epsilon_{t+1}(w_t)$, we get by apply Proposition \ref{prop:convergence_stationary_point},
\[ \frac{1}{2}\|w_0 - v_*\|^2 + \frac{1}{2}\eta^2 \sigma_3^2 T
\leq \sum_{t=0}^T \bE\left[\| v_{t} - v_* \|^2 \right] 
+ \frac{4}{3}\eta\bE[f(w_0)] + \frac{2}{3}\eta^3 \sigma_1^2 L (T+2)
+ \eta^2 \| \nabla f(w_0)\|^2. \]
Therefore, we have 
\begin{align*}
    \frac{1}{4}\eta^2 \sigma_3^2 - \frac{4}{3}\eta^3 \sigma_1^2 L
    &\leq \frac{1}{T+1}\sum_{t=0}^T \bE\left[\| v_{t} - v_* \|^2 \right] 
    + \frac{1}{T+1}\left( \frac{4}{3}\eta\bE[f(w_0)] + \eta^2 \| \nabla f(w_0)\|^2 - \frac{1}{2}\|w_0 - v_*\|^2 \right) \\
    &= \frac{1}{T+1}\sum_{t=0}^T \bE\left[\| v_{t} - v_* \|^2 \right] 
    + O(T^{-1}).
\end{align*}
Since $\eta \leq \frac{3\sigma_3^2}{32\sigma_1^2 L}$, we have $\frac{1}{8}\eta^2 \sigma_3^2 \leq \frac{1}{4}\eta^2 \sigma_3^2 - \frac{4}{3}\eta^3 \sigma_1^2 L$. Therefore, the above inequality concludes the proof.
\end{proof}

\section{Motivating Example 1}\label{sec:example}
We present an example that demonstrates the implicit bias of SGD and averaged SGD. Moreover, we theoretically and empirically verify that averaged SGD can get closer to a flat minimum than SGD.
\subsection{Problem Setup}
In this section, we present a motivating example that verifies the convergence to a flat minimum and a certain separation between SGD and averaged SGD.
We consider a one-dimensional objective function $f: \bR \rightarrow \bR$ defined below: for $\delta > 0$, 
\begin{equation}\label{eq:example}
f(w) = \frac{1}{2}(w-1)^2 + g_\delta(w),
\end{equation}
where $g_\delta: \bR \rightarrow \bR$ is a scaled mollifier: 
\begin{align*}
g_\delta(w) = 
\left\{
\begin{array}{ll}
- \delta \exp \left( 1 - \frac{1}{1-\left(\frac{w}{\delta}\right)^2} \right) & (|w| < \delta), \\
0 & (|w| \geq \delta).
\end{array}
\right.
\end{align*}
$g_\delta(w) = \delta g_1(w/\delta)$ is a scaling of the well-known mollifier of $g_1$ which is an infinitely differentiable function with a compact support.
That is, $g_\delta$ is a smooth function whose support is $[-\delta,\delta]$.
Note that the function $f(w)$ has a local minimum in $[-\delta,\delta]$, which becomes sharp when $\delta$ is small. 
See Figure \ref{fig:example} which depicts graphs of $g_\delta,~G_\delta,~f,$ and $F$.

\begin{figure}[t]
\center
\includegraphics[width=0.9\textwidth]{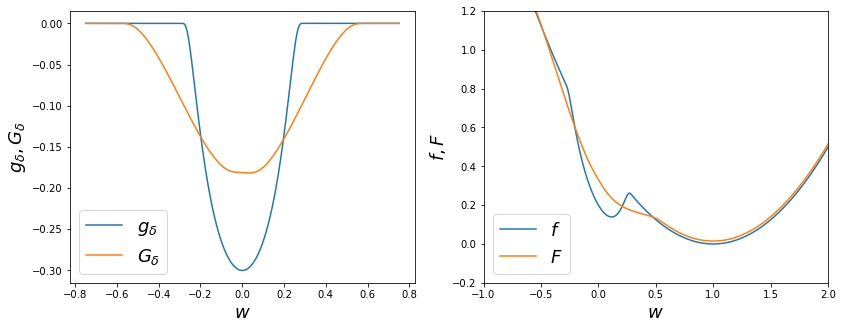} \\
\caption{The left figure plots the mollifier $g_\delta$ (blue) and smoothed mollifier $G_\delta$ (orange), and the right figure plots the objective $f$ (blue) and smoothed objective $F$ (orange). Hyperparameters are set to $\delta=0.3,~r=1,$ and $\eta=0.3$.}
\label{fig:example}
\end{figure}

The maximum values of the first and second derivatives of $g_1$ are bounded. Thus, we define constants $C_1, C_2$ by
\[ C_1 = \max \left\{1 , \max_w |g'_1(w)|\right\},~ C_2 = \max_w |g''_1(w)|. \]

Since $g''_\delta(w) = \frac{1}{\delta} g''_1(w/\delta)$, we see the second derivative of $g_\delta$ is bounded by $C_2 \delta^{-1}$. 
Hence, Lipschitz smoothness (boundedness of Hessian) $L$ of $f$ is $1 + C_2 \delta^{-1}$.

Next, we consider the uniform noise on the interval $[-r,r]$ for $r>0$, i.e., $\epsilon \sim U[-r,r]$ and suppose $\epsilon (w,z)=\epsilon(z) (=\epsilon'(v,z))$ where $\Omega \ni z \mapsto \epsilon(w,z)$ is an explicit representation of the random noise. 
In other words, noise distribution does not change in $w$.
In this case, we see $\sigma_1^2 = \sigma_3^2 = \bE[\epsilon^2 ]\leq \frac{r^2}{3}$ and $\sigma_2 = 0$.
The smoothed objective $F$ with the noise $\epsilon'$ and step-size $\eta$ is 
\begin{align*}
F(v) 
&= \bE[ f(v-\eta \epsilon') ] \\
&= \frac{1}{2}(v-1)^2 + \eta^2 \sigma_1^2 + \bE[ g_\delta( v - \eta \epsilon' ) ] \\
&\sim \frac{1}{2}(v-1)^2 + \bE[ g_\delta( v - \eta \epsilon' ) ] .
\end{align*}

We consider the following problem setup:
\begin{align}
&\delta \leq \frac{1}{4 (1+2C_1)},  \label{example_cond1}\\
&r \geq \frac{64}{3}C_1 (\delta + C_2).  \label{example_cond2}
\end{align}
Note that we can choose arbitrarily small $\delta>0$ and large $r$ which satisfy the above inequalities.
Under these conditions, we can choose the following step size $\eta$ for appropriate smoothing.
\begin{equation}\label{step_size_cond}
\frac{2 C_1 \delta}{r} \leq \eta \leq \min\left\{ \frac{1}{4r} - \frac{\delta}{r}, \frac{3\delta}{32(\delta + C_2)} \right\}.
\end{equation}
We note that the above step size satisfies $\eta \leq \min\left\{ \frac{1}{2L},~\frac{3\sigma_3^2}{32\sigma_1^2 L}\right\}$ required in Propositions \ref{prop:sgd_upper_bound} and \ref{prop:sgd_lower_bound} since $\frac{3\delta}{32(\delta + C_2)} = \frac{3}{32L} =  \frac{3\sigma_3^2}{32\sigma_1^2L} < \frac{1}{2L}$.

\subsection{Evaluation of SGD and Averaged SGD}
Under the above setup (\ref{example_cond1})--(\ref{step_size_cond}), we can estimate constants appearing in the convergence results of SGD and averaged SGD as follows (for the detail see the next subsection):
\begin{align}\label{eq:estimated_constants}
  L = 1 + \frac{C_2}{\delta},~\sigma_1^2=\sigma_3^2 = \frac{r^2}{3},~\sigma_2 = 0, 
  \mu=1,~c=\frac{1}{3},~\gamma=0,~M=\frac{8}{9}.
\end{align}
Moreover, the minimum of the smoothed objective is $v_*=1$, and a sharp minimum ($\sim 0$) can be eliminated by smoothing.

Therefore, for SGD we obtain by Proposition \ref{prop:sgd_upper_bound},
\begin{align*}
    D_T^2 &= \frac{1}{T+1}\sum_{t=0}^T \bE[\|  v_t - v_* \|^2]\\
    &\leq O\left(T^{-1}\right) + \frac{2\eta\sigma_1^2}{c} + \frac{8\eta^2 \sigma_1^2 L}{3c} \left(1 + \frac{2 \eta \sigma_2^2}{c} \right) \\
    &\leq O\left(T^{-1}\right) + 6\eta r^2 + 8\eta^2 r^2 \left(1 + \frac{C_2}{\delta} \right).
\end{align*}
We see from this inequality, $\eta_*=\frac{2C_1 \delta}{r}$ is the best choice of the step-size, resulting in
\begin{align*}
    D_T \leq O\left(T^{-1/2}\right) + \sqrt{12 C_1 \delta r + 32 C_1^2 \delta \left( \delta + C_2 \right)},    
\end{align*}
and 
\[ \frac{C_1 \delta}{\sqrt{6}} \leq D_\infty \leq 2\sqrt{ (3 r + 8C_1 C_2 ) C_1 \delta} + 4\sqrt{2}C_1 \delta, \]
where we applied Proposition \ref{prop:sgd_lower_bound} for the lower bound.
This result means SGD avoids a sharp minimum (i.e., $v\sim 0$ under small $\delta>0$) and converges to a flat minimum $v_*=1$, and a too large noise will affect the convergence to $v_*$.

Hence, Eq.~(\ref{eq:improvement_condition}): $0 \ll D_\infty \ll 9/8$ is satisfied under the above setting with mildly small $\delta$, meaning that the error of averaged SGD is strictly smaller than that of SGD since $\lim_{T\rightarrow \infty}\bE[\|\overline{v}_T-v_*\|] \leq \frac{8}{9}D_\infty^2$ by Eq.~(\ref{eq:asgd_upper_bound_limit}).
We empirically observed this phenomenon in Figure \ref{fig:example_sgd_asgd} in which we run SGD and averaged SGD for problems with small $\delta = 0.1$ and relatively large $\delta = 0.3$.

\begin{figure}[t]
\center
\includegraphics[width=0.9\textwidth]{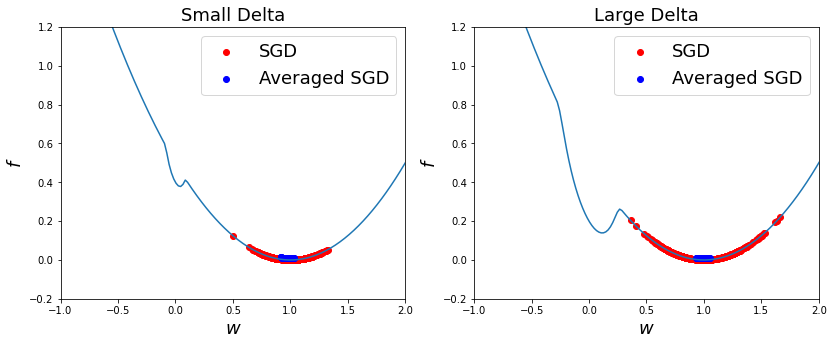} \\
\caption{The figures plot the convergent points of SGD and averaged SGD for the problem (\ref{eq:example}). For the left case, we set $\delta=0.1,~r=1,~\eta=\delta/r = 0.1$ and for the right case, we set $\delta=0.3,~r=1,~\eta=\delta/r = 0.3$.
}
\label{fig:example_sgd_asgd}
\end{figure}

\subsection{Estimation of Constants}\label{subsec:estimate_constants}
We verify the estimations of constants in (\ref{eq:estimated_constants}).
$L$, $\sigma_1^2$, and $\sigma_2$ are already obtained, thus, $\mu, c$, and $M$ remain.

\paragraph{Minimum and estimation of $\mu$.} We first see that under our problem setting, the local minimum around the origin is eliminated and $1$ is the optimal solution of $F$, i.e., $v_* = 1$.

The smoothed function $G_\delta(v) \defeq \bE[ g_\delta( v - \eta \epsilon' ) ]$ and its derivative $G'_\delta(v)$ are calculated as follows:
\begin{align*}
 G_\delta(v) &= \int_{-r}^r g_\delta( v - \eta t ) \frac{1}{2r} \mathrm{d}t, \\
 G'_\delta(v) &= \int_{-r}^r g'_\delta( v - \eta t ) \frac{1}{2r} \mathrm{d}t.
\end{align*}
By taking into account $\mathrm{supp}(g_\delta) = [-\delta, \delta]$, the smoothed objective $G_\delta(v)$ is constant on $\{ |v| \leq \eta r - \delta \} \cup \{ |v| \geq \eta r + \delta\}$, and thus, $G_\delta'$ is non-zero only on $\mathrm{supp}(G'_\delta) = [-\eta r - \delta, -\eta r + \delta] \cup [\eta r - \delta, \eta r + \delta]$. For graphs of $g_\delta$ and $G_\delta$, see Figure \ref{fig:example} (left).
Since $\eta r + \delta < 1/4$ under (\ref{step_size_cond}), $G_\delta$ is a constant around $v=1$, that is, $v=1$ is still a local minimum of $F$.

We evaluate the bound on $G'_\delta$ on $\mathrm{supp}(G'_\delta)$ below. For $v \in [\eta r - \delta, \eta r + \delta]$ the support of $g'_\delta( v - \eta t)$ in $t \in \bR$ is $[ (v-\delta)/\eta, (v+\delta)/\eta]$, we get
\begin{align*}
  0 \leq G'_\delta(v)
  &= \int_{-r}^r g'_\delta( v - \eta t ) \frac{1}{2r} \mathrm{d}t \\
  &\leq \int_{-\infty}^\infty |g'_\delta( v - \eta t )| \frac{1}{2r} \mathrm{d}t \\  
  &= \int_{\frac{v-\delta}{\eta}}^{\frac{v+\delta}{\eta}} |g'_\delta( v - \eta t )| \frac{1}{2r} \mathrm{d}t \\
  &\leq C_1\int_{\frac{v-\delta}{\eta}}^{\frac{v+\delta}{\eta}}  \frac{1}{2r} \mathrm{d}t
  = \frac{C_1\delta}{\eta r},
\end{align*}
where we used $|g'_\delta (v)|=|g'_1(v/\delta)| \leq C_1$.
A bound on $[-\eta r - \delta, -\eta r + \delta]$ is also obtained in the same way. Thus, we see
\begin{align*}
  \left\{
\begin{array}{rll}
  - \frac{C_1\delta}{\eta r} \leq &G'_\delta(v) \leq 0 & (v \in [-\eta r - \delta, -\eta r + \delta]), \\
  0 \leq &G'_\delta(v) \leq \frac{C_1\delta}{\eta r} & (v \in [\eta r - \delta, \eta r + \delta]), \\
  ~&G'_\delta(v) = 0 & (\textrm{else}). \\
\end{array}
\right.
\end{align*}

If there are additional stationary points of $F$, they should exist in $[\eta r - \delta, \eta r + \delta] = \mathrm{supp}(G'_\delta) \setminus [-\eta r - \delta, -\eta r + \delta]$ because of the sign of $G'_\delta$ and $\mathrm{supp}(G'_\delta) \subset (-\infty, 1/4)$.
However, since $\eta r + \delta \leq 1/4$ and $\frac{C_1\delta}{\eta r} \leq 1/2$ under (\ref{step_size_cond}), we see
\[ \max_{v \in [\eta r - \delta, \eta r + \delta]} F'(v) \leq (\eta r + \delta) - 1 + \frac{C_1\delta}{\eta r} \leq \frac{1}{4} - 1 + \frac{1}{2} = - \frac{1}{4}. \]
Hence, $v_* = 1$ is the unique local minimum (i.e., optimal solution) of $F$ and we can conclude $\mu = 1$.

\paragraph{Estimation of $c$.}
From the above argument, we get
\begin{align*}
  F'(v)(v-1) &= (v-1)^2 + G'_\delta(v)(v-1) \\
  &\geq \left\{
\begin{array}{ll}
  (v-1)^2  & (v \in [-\eta r - \delta, -\eta r + \delta]), \\
  (v-1)^2 +  \frac{C_1\delta}{\eta r} (v-1) \geq (v-1)^2 + \frac{1}{2} (v-1) & (v \in [\eta r - \delta, \eta r + \delta]), \\
  (v-1)^2  & (\textrm{else}). \\
\end{array}
\right.
\end{align*}
Clearly, $1/2 \leq 2(1-v)/3$ for $v \leq \eta r + \delta \leq 1/4$.
Thus, $F'(v)(v-1) \geq (v-1)^2/3$ on $v \in [\eta r - \delta, \eta r + \delta]$ and we conclude $c=1/3$.

\paragraph{Estimation of $\gamma$ and $M$.}
Noting $v_*=1$ and $F''(1) = 1$, we have
\begin{align*}
  | F'(v) - F''(v_*)(v-v_*)|
  = | (v-1) + G'_\delta(v) - (v-1)| = |G'_\delta(v)|.
\end{align*}
Because of the problem setup, it is enough to verify $\gamma=0,~M=\frac{8}{9}$ satisfy $|G'_\delta(v)| \leq M |v-1|^2$ on $v \in [\eta r - \delta, \eta r + \delta]$.
Since $|G'_\delta(v)| \leq 1/2$ and $v \leq 1/4$ for $v$ in this interval, we have
\[ |G'_\delta(v)| \leq \frac{1}{2} \leq M \frac{9}{16} \leq M (v-1)^2. \]
This concludes $\gamma=0,~M=\frac{8}{9}$.

\section{Motivating Example 2}\label{sec:example2}
We provide an example for which SGD and averaged SGD behave in significantly different ways.
For this example, the convergence of SGD cannot be guaranteed whereas averaged SGD can converge.
\subsection{Problem Setup}
We consider a one-dimensional objective function $f: \bR \rightarrow \bR$ defined below: for $\delta > 0$, 
\begin{equation}\label{eq:example2}
f(w) = \frac{1}{2}w^2 + g_{\delta}(w),
\end{equation}
where $g_\delta: \bR \rightarrow \bR$ is a scaled mollifier: 
\begin{align*}
g_\delta(w) = 
\left\{
\begin{array}{ll}
\delta^2 \exp \left( 1 - \frac{1}{1-\left(\frac{w}{\delta}\right)^2} \right) & (|w| < \delta), \\
0 & (|w| \geq \delta).
\end{array}
\right.
\end{align*}

The function $f$ is symmetry and has two local minima. 
Taking the derivative at $v = \alpha \delta$ ($\alpha > 0$), we see
\[ f'(\alpha \delta) = \alpha\delta \left( 1 - \frac{2}{(1-\alpha^2)^2}\exp\left( 1- \frac{1}{1-\alpha^2}\right) \right). \]
Therefore, the value of $\alpha >0$ such that $f'(\alpha \delta) = 0$ is independent of $\delta$, meaning optimal solutions of $f$ are $O(\delta)$.

Let $\epsilon \sim U[-r,r]$ be the uniform distribution as in Section \ref{sec:example}.
Then, the smoothed objective $F$ with the noise $\epsilon'$ and step-size $\eta$ is 
\begin{align*}
F(v) 
&= \frac{1}{2}v^2 + \eta^2 \sigma_1^2 + \bE[ g_\delta( v - \eta \epsilon' ) ] \\
&\sim \frac{1}{2}v^2 + \bE[ g_\delta( v - \eta \epsilon' ) ] .
\end{align*}

Figure \ref{fig:example2} depicts the above functions. We note that $g_\delta$ is slightly different from that in the previous section.
\begin{figure}[t]
\center
\includegraphics[width=0.9\textwidth]{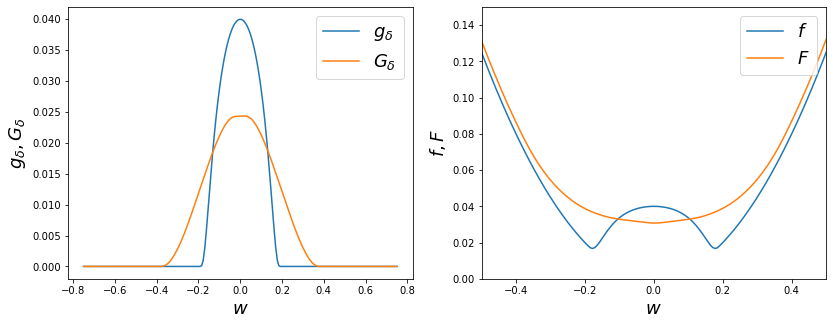} \\
\caption{The left figure plots the mollifier $g_\delta$ (blue) and smoothed mollifier $G_\delta$ (orange), and the right figure plots the objective $f$ (blue) and smoothed objective $F$ (orange). Hyperparameters are set to $\delta=0.2,~r=1,$ and $\eta=0.2$.}
\label{fig:example2}
\end{figure}

We define constants $C_1, C_2$ as follows:
\[ C_1 = \max_w |g'_1(w)|,~ C_2 = \max\{1, \max_w |g''_1(w)|\}. \]
For sufficient smoothing, we consider the following problem setup:
\begin{align}
C_1 \delta^2 &\leq \frac{\eta r}{2} (\eta r - \delta),  \label{example2_cond1}\\
C_2 \delta &\leq \frac{\eta r}{2}.  \label{example2_cond2}
\end{align}
Note that for a given $\eta r > 0$, the above conditions can be satisfied if $\delta>0$ is sufficiently small.

Lipschitz smoothness $L$ of $f$ is $L=1 + C_2$, and hence $\eta$ should satisfy $\eta \leq \min\left\{ \frac{3\sigma_3^2}{32\sigma_1^2 L}, \frac{1}{2L} \right\} = \frac{3}{32L} = \frac{3}{32(1+C_2)}$.
We note contrary to the case in the previous section, $L$ does not suffer from $\delta$.
Hence we do not need to make $\eta$ small depending on $\delta$. 

\subsection{Evaluation of SGD and Averaged SGD}
Under conditions (\ref{example2_cond1}) and (\ref{example2_cond2}), we can estimates constants appearing in the convergence results of SGD and averaged SGD as follows (for the detail see the next subsection):

\begin{align}\label{eq:estimated_constants2}
  L = 1 + C_2,~\sigma_1^2=\sigma_3^2=\frac{r^2}{3},~\sigma_2 = 0, 
  \mu=1,~c=\frac{1}{2},~\gamma=\frac{C_1\delta^2}{\eta r},~M=0.
\end{align}
Moreover, the smoothed objective $F$ has a unique solution $v_*=0$ as shown in the next subsection.

Therefore, for SGD we obtain by Proposition \ref{prop:sgd_upper_bound},
\begin{align*}
    D_T^2 &
    = \frac{1}{T+1}\sum_{t=0}^T \bE[\|  v_t - v_* \|^2]\\
    &\leq O\left(T^{-1}\right) + \frac{2\eta\sigma_1^2}{c} + \frac{8\eta^2 \sigma_1^2 L}{3c} \left(1 + \frac{2 \eta \sigma_2^2}{c} \right) \\
    &\leq O\left(T^{-1}\right) + 4\eta r^2 + \frac{16}{3}\eta^2 r^2 (1 + C_2).
\end{align*}
That is,
\begin{align}\label{eq:example2_sgd}
    D_T  \leq O\left(T^{-1/2}\right) + 2\sqrt{\eta r^2 + \frac{4}{3}\eta^2 r^2 (1 + C_2)},
\end{align}
and 
\[ \frac{\eta r}{2\sqrt{6}} \leq D_\infty \leq 2\sqrt{\eta}r + 4\eta r \sqrt{\frac{1+C_2}{3}}, \]
where we applied Proposition \ref{prop:sgd_lower_bound} for the lower bound.
Taking into account (\ref{example2_cond2}), we see the right hand side of this bound is $\Omega(\delta)$.
Then, since the optimal solution of $f$ is also $O(\delta)$, we can see that upper-bound (\ref{eq:example2_sgd}) on SGD cannot distinguish solutions of the original function $f$ and the smoothed objective $F$.

On the other hand, the condition Eq.~(\ref{eq:improvement_condition}): $C_1\delta^2/(\eta r) \ll D_\infty \ll \infty$ is satisfied under the setting $\delta^2 \ll \eta^2 r^2$. Then, we have $\lim_{T\rightarrow \infty}\bE[\|\overline{v}_T-v_*\|] \leq C_1\delta^2/(\eta r)$ by by Eq.~(\ref{eq:asgd_upper_bound_limit}). Since, $\delta$ can be small independent of $\eta$ and $r$, averaged SGD can approach a solution $v_*=0$ of $F$ up to $O(\delta^2)$. 

This difference between SGD and averaged SGD is intuitively understandable. In the case of using a large step size compared to $\delta$ (i.e., $\delta^2 \ll \eta^2 r^2$), SGD does not converge to a point and hence oscillates in the valley. Then, its mean should be located almost at the origin (i.e., $v_*$), meaning the convergence of averaged SGD. In the case of using a small step size, both SGD and averaged SGD do not converge to $v_*$ but converge to a minimizer of $f$. 
Indeed, we empirically observed this phenomenon in Figure \ref{fig:example2_sgd_asgd}. Both SGD and averaged SGD using small step size $\eta=0.01$ compared to $\delta$ converge to the solutions of $f$ because of the weakened bias and precise optimization. On the other hand, when using a relatively large step size $\eta=0.2$, SGD cannot converge, whereas averaged SGD converges to the solution $v_*=0$ of $F$ very accurately.

\begin{figure}[t]
\center
\includegraphics[width=0.9\textwidth]{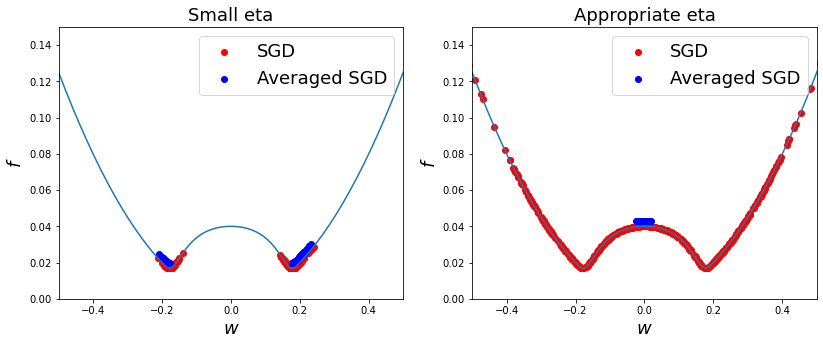} \\
\caption{The figures plot the convergent points of SGD and averaged SGD using small step size $\eta=0.01$ (left) and appropriately large $\eta=0.2$ (right). For both cases, we set $\delta = 0.2$ and $r = 1$.}
\label{fig:example2_sgd_asgd}
\end{figure}

\subsection{Estimation of Constants}
\paragraph{Minimum and estimation of $\mu$.}
In a similar way to Section \ref{sec:example}, the smoothed function $G_\delta(v) = \bE[ g_\delta(v-\eta t)]$ and its derivative $G'_\delta$ satisfy the following. 
By taking into account $\mathrm{supp}(g_\delta) = [-\delta, \delta]$, the smoothed objective $G_\delta(v)$ is constant on $\{ |v| \leq \eta r - \delta \} \cup \{ |v| \geq \eta r + \delta\}$, and thus, $G_\delta'$ is non-zero only on $\mathrm{supp}(G'_\delta) = [-\eta r - \delta, -\eta r + \delta] \cup [\eta r - \delta, \eta r + \delta]$. 
Specifically, we obtain
\begin{align*}
  \left\{
\begin{array}{rll}
  0 \leq &G'_\delta(v) \leq \frac{C_1\delta^2}{\eta r} & (v \in [-\eta r - \delta, -\eta r + \delta]), \\
  - \frac{C_1\delta^2}{\eta r} \leq &G'_\delta(v) \leq 0 & (v \in [\eta r - \delta, \eta r + \delta]), \\
  ~&G'_\delta(v) = 0 & (\textrm{else}). 
\end{array}
\right.
\end{align*}

It is verified that the derivative of $F$ vanishes only at $v=0$ as follows. 
If there is a point so that $F'(v)=0$ other than $v\neq 0$, $v$ should be in $\mathrm{supp}(G_\delta)$.
For $v \in [\eta r - \delta, \eta r + \delta]$, we get 
\[ F'(v) = v + G'_\delta(v) \geq \eta r - \delta - \frac{C_1 \delta^2}{\eta r} > 0, \]
where we used (\ref{example2_cond1}).
In a similar way, we have $F'(v) < 0$ on $v \in [-\eta r - \delta, -\eta r + \delta]$. 
Moreover, $F'(0)=0$ clearly holds. Thus, $v_*=0$ is a unique solution.

Since $G_\delta$ is constant on $[-\eta r + \delta, \eta r - \delta]$ under the condition (\ref{example2_cond2}), we see $G''_\delta(0)=0$.
Therefore, we get $F''(0)=v$, and hence $\mu=1$.

\paragraph{Estimation of $c$.}
From the above argument, we get
\begin{align*}
  F'(v)v &= v^2 + G'_\delta(v)v \\
  &\geq \left\{
\begin{array}{ll}
  v^2 + \frac{C_1 \delta^2}{\eta r}v  & (v \in [-\eta r - \delta, -\eta r + \delta]), \\
  v^2 - \frac{C_1 \delta^2}{\eta r}v  & (v \in [\eta r - \delta, \eta r + \delta]), \\
  v^2  & (\textrm{else}). \\
\end{array}
\right.
\end{align*}

Since $\frac{C_1 \delta^2}{\eta r} \leq \frac{\eta r - \delta}{2} \leq -\frac{v}{2}$ on $[-\eta r -\delta, -\eta r + \delta]$ and $\frac{C_1 \delta^2}{\eta r} \leq \frac{\eta r - \delta}{2} \leq \frac{v}{2}$ on $[\eta r -\delta, \eta r + \delta]$, 
we have $F'(v)v \geq \frac{v^2}{2}$. Thus, $c = \frac{1}{2}$.

\paragraph{Estimation of $\gamma$ and $M$.}
Since $G''_\delta(0)=0$, we have $F'(v) - F''(0)v = G'_\delta(v)$ for any $v \in \bR$, leading to $|F'(v) - F''(0)v| = |G'_\delta(v)| \leq \frac{C_1 \delta^2}{\eta r}$.
Hence, $\gamma = \frac{C_1 \delta^2}{\eta r}$ and $M=0$.

\section{Additional Experiments}\label{sec:add_exp}
\begin{table}[th]
\caption{Comparison of test classification accuracies on CIFAR10 dataset. All methods adopt the multi-step strategy for the step size schedule.}
\label{table:comparison_on_cifar10_table}
\begin{center}
\begin{footnotesize}
  \begin{tabular}{cccc}
    \toprule
~&
\multicolumn{3}{c}{CIFAR10} 
\\
\midrule
~&
$\eta$ &
ResNet-50 &
WRN-28-10 
\\
\midrule
SGD & 
\textit{s} &
95.95 \scriptsize{(0.10)} &
96.85 \scriptsize{(0.16)} 
\\
\addlinespace
\multirow{3}{*}{\begin{tabular}{c} Averaged\\SGD \end{tabular}}  &  
\textit{s} &
96.58 \scriptsize{(0.14)} &
97.24 \scriptsize{(0.07)} 
\\
&
\textit{m} &
\textbf{96.89 \scriptsize{(0.05)}}&
\textbf{97.44 \scriptsize{(0.04)}}
\\
&
\textit{l} &
96.27 \scriptsize{(0.16)} &
97.05 \scriptsize{(0.09)} 
\\
\bottomrule
\end{tabular}
\end{footnotesize}
\end{center}
\end{table}

We run SGD and averaged SGD on CIFAR10 dataset with the step size strategy \textit{`l'} under the same settings as in Section \ref{sec:exp}.
Table \ref{table:comparison_on_cifar10_table} lists the results including this case. 
We observe that the large step size \textit{`l'} does not work so well on CIFAR10 dataset compared to other schedules.
We hypothesize this is because CIFAR10 is not so difficult dataset and does not require stronger bias induced by a larger step size.

We also validate the cosine annealing strategy for the step size, which is frequently used due to its excellent performance.
We used the symbols \textit{`s'}, \textit{`m'}, and \textit{`l'} for the cosine annealing depending on the last step sizes which are set to $0$, $0.004$, and $0.02$, respectively.
The parameter averaging for averaged SGD is taken over the last quarter of the training.
From the table, we observe the usefulness of parameter averaging for the cosine annealing schedule as well.

\begin{table}[th]
\caption{Comparison of test classification accuracies on CIFAR100 and CIFAR10 datasets. All methods adopt cosine annealing for the step-size schedule.}
\label{comparison_table2}
\begin{center}
\begin{footnotesize}
  \begin{tabular}{ccccccccc}
    \toprule
~&
~&
\multicolumn{3}{c}{CIFAR100} &
~&
\multicolumn{3}{c}{CIFAR10} 
\\
\midrule
~&
$\eta$ &
ResNet-50 &
WRN-28-10 &
Pyramid &
$\eta$ &
ResNet-50 &
WRN-28-10 &
Pyramid
\\
\midrule
SGD & 
\textit{s} &
82.26 &
82.68 &
82.97 &
\textit{s} &
96.58 &
97.00 &
96.66
\\
\addlinespace
\multirow{2}{*}{\begin{tabular}{c} Averaged\\SGD \end{tabular}}  &  
\textit{s} &
\textbf{83.89} &
84.28 &
85.14 &
\textit{s} &
\textbf{97.01} &
97.28 &
97.07
\\
&
\textit{l} &
83.21 &
\textbf{84.49} &
\textbf{85.47} &
\textit{m} &
96.86 &
\textbf{97.51} &
\textbf{97.32}
\\
\midrule
SAM &  
\textit{s} &
83.35 &
84.64 &
86.24 &
\textit{s} &
96.40 &
96.89 &
97.61
\\
\addlinespace
\multirow{2}{*}{\begin{tabular}{c} Averaged\\SAM \end{tabular}}  &  
\textit{s} &
83.18 &
84.94 &
86.79 &
\textit{s} &
\textbf{96.56} &
97.14 &
\textbf{97.55}
\\
&
\textit{l} &
\textbf{83.58} &
\textbf{85.26} &
\textbf{86.84} &
\textit{m} &
96.51 &
\textbf{97.19} &
97.48
\\
\bottomrule
\end{tabular}
\end{footnotesize}
\end{center}
\end{table}

\begin{figure}[t]
\center
\begin{tabular}{cc}
\begin{minipage}[t]{0.42\linewidth}
\centering
\includegraphics[width=\textwidth]{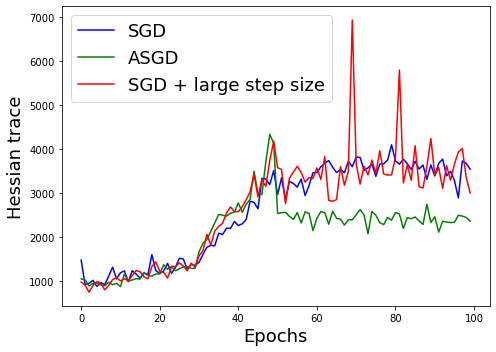} \\
\end{minipage} 
\hspace{3mm}
\begin{minipage}[t]{0.42\linewidth}
\centering
\includegraphics[width=\textwidth]{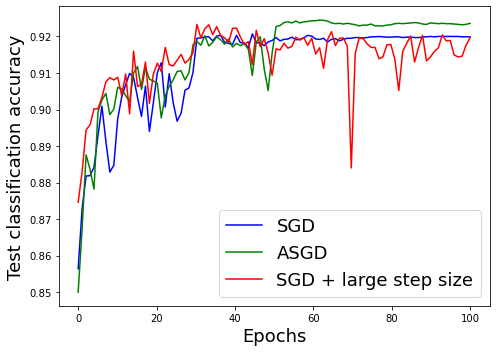} \\
\end{minipage} 
\end{tabular}
\caption{The figure depicts the curve of the trace of Hessian $\nabla^2 f(w)$ and test loss functions achieved by SGD, SGD with large step size, and averaged SGD. Each algorithm is run to train the standard convolutional neural network on Fashion MNIST dataset.}
\label{fig:fashin_mnist}
\end{figure}

Finally, we run SGD, SGD with a large step size, and averaged SGD to train the standard convolutional neural network on Fashion MNIST dataset to confirm how efficiently sharpness and classification accuracy can be optimized by each method.
We note the large step size used for SGD is the same as that for averaged SGD.
We plot the trace of Hessian $\nabla^2 f(w)$ and test loss functions in Figure \ref{fig:fashin_mnist}.
From this figure, we observe that the averaged SGD converges to a flatter region and achieves the highest classification accuracy on the test dataset as expected in our theory.

\section{Related Literature and Discussion}
\paragraph{Flat Minimum.}
\citet{keskar2017large} and \citet{hochreiter1997flat} showed a flat minimum generalizes well and a sharp minimum generalizes poorly. However, \cite{dinh2017sharp} pointed out that the flatness solely cannot explain generalization because it can be easily manipulated.
\citet{neyshabur2017exploring} rigorously proved the sharpness combined with $\ell_2$-norm provides a generalization bound and \citet{jiang2020fantastic} confirmed this strong correlation through large scale experiments.
\citet{keskar2017large} also argued that SGD converges to a flat minimum and \citet{he2019asymmetric} argued that the averaged SGD tends to converge to an asymmetric valley due to the stochastic gradient noise. 
Indeed, several works \citep{kleinberg2018alternative,zhou2020towards} justified the implicit bias of SGD towards a flat region or asymmetric valley.
Moreover, \cite{izmailov2018averaging,foret2020sharpness,damian2021label,orvieto2022explicit,kaddour2022} studied the techniques to further bring out the bias of SGD with improved performance.
In particular, SAM and SWA achieved a significant improvement in generalization.
In our paper, we show that parameter averaging stabilizes the convergence to a flat region or asymmetric valley, and suggest the usefulness of the combination with the large step size for the difficult dataset which needs a stronger regularization.
Besides, several authors proposed methods that explicitly inject noise for improving generalization \citep{chaudhari2019entropy}, in particular for the large batch setting \citep{wen2018smoothout,haruki2019gradient,lin2020extrapolation}.

\paragraph{Markov Chain Interpretation of SGD.}
\citet{dieuleveut2020,yu2020analysis} provided the Markov chain interpretation of SGD.
They showed the marginal distribution of the parameter of SGD converges to an invariant distribution for convex and nonconvex optimization problems, respectively.
Moreover, \citet{dieuleveut2020} showed the mean of the invariant distribution, attained by the averaged SGD, is at distance $O(\eta)$ from the minimizer of the objective function, whereas SGD itself oscillates at distance $O(\sqrt{\eta})$ in the convex optimization settings.
\citet{izmailov2018averaging} also attributed the success of SWA to such a phenomenon.
That is, \citet{izmailov2018averaging} explained that SGD travels on the hypersphere because of the convergence to Gaussian distribution and the concentration on the sphere under a simplified setting, and thus averaging scheme allows us to go inside of the sphere which may be flat.
We can say our contribution is to theoretically justify this intuition by extending the result obtained by \citet{dieuleveut2020} to a nonconvex optimization setting.
In the proof, we utilize the alternative view of SGD \citep{kleinberg2018alternative} in a non-asymptotic way under some conditions not on the original objective but on the smoothed objective function.
Combination with the Markov chain view for nonconvex objective \citep{yu2020analysis} may be helpful in more detailed analyses.

\paragraph{Step size and Minibatch.}
SGD with a large step size often suffers from stochastic gradient noise and becomes unstable. This is the reason why we should take a smaller step size so that SGD converges. In this sense, the minibatching of stochastic gradients clearly plays the same role as the step size and sometimes brings additional gains. For instance, \citet{smith2017don} empirically demonstrated that the number of parameter updates can be reduced, maintaining the learning curves on both training and test datasets by increasing minibatch size instead of decreasing step size.
We remark that our analysis can incorporate the minibatch by dividing $\sigma_1^2$ and $\sigma_2^2$ in Proposition \ref{prop:sgd_upper_bound} and \ref{thm:asgd} by the minibatch size, and we can see certain improvements of optimization accuracy as well. 

\paragraph{Edge of Stability.}
Recently, \citet{cohen2020gradient} showed the deterministic gradient descent for deep neural networks enters \text{Edge of Stability} phase.
In the traditional optimization theory, the step size is set to be smaller than $1/L$ to ensure stable convergence and we also make such a restriction.
On the other hand, the Edge of Stability phase appears when using a higher step size than $2/L$.
In this phase, the training loss behaves non-monotonically and the sharpness finally stabilizes around $2/\eta$.
This can be explained as follows \citep{lewkowycz2020large}; if the sharpness around the current parameter is large compared to the step size, then gradient descent cannot stay in such a region and goes to a flatter region that can accommodate the large step size.
There are works \citep{pmlr-v162-arora22a,pmlr-v162-ahn22a} which attempted to rigorously justify the Edge of Stability phase.
Interestingly, their analyses are based on a similar intuition to ours, but we consider a different regime of step sizes and a different factor (stochastic noise or larger step size than $2/L$) brings the implicit bias towards flat regions. We believe establishing a unified theory is interesting future research.


\paragraph{Averaged SGD.}
The averaged SGD \citep{ruppert1988efficient,polyak1992acceleration} is a popular variant of SGD, which returns the average of parameters obtained by SGD aiming at stabilizing the convergence.
Because of the better generalization performance, many works conducted convergence rate analysis in the expected risk minimization setting and derived the asymptotically optimal rates $O(1/\sqrt{T})$ and $O(1/T)$ for non-strongly convex and strongly convex problems \citep{nemirovski2009robust,bach2011non,rakhlin2012making,lacoste2012simpler}. However, the schedule of step size is basically designed to optimize the original objective function, and hence the implicit bias coming from the large step size will eventually disappear.
When applying a non-diminishing step size schedule, the non-zero optimization error basically remains.
What we do in this paper is to characterize it as the implicit bias toward a flat region.

\subsection{Technical difference.}
The proof idea of Proposition \ref{prop:sgd_upper_bound} relies on the alternative view of SGD \citep{kleinberg2018alternative} which shows the existence of an associated SGD for the smoothed objective. However, since its stochastic gradient is a biased estimator, they showed the convergence not to the solution but to a point at which a sort of one-point strong convexity holds, and avoid the treatment of a biased estimator. Hence, the optimization of the smoothed objective is not guaranteed in their theory. On the other hand, optimization accuracy is the key in our theory, thus we need nontrivial refinement of the proof under a {\it normal} one-point strong convexity at the solution.

}

\end{document}